\documentclass[10pt,a4paper]{article}
\usepackage[margin=0.95in]{geometry}
\usepackage{amsmath,amssymb,amsthm}
\usepackage{mathtools}
\usepackage{booktabs}
\usepackage{multirow}
\usepackage{graphicx}
\usepackage[ruled,linesnumbered]{algorithm2e}
\usepackage{hyperref}
\usepackage[capitalize]{cleveref}
\usepackage{microtype}
\usepackage{xcolor}
\usepackage{enumitem}
\usepackage{caption}
\usepackage{subcaption}
\usepackage{float}
\usepackage{bm}
\usepackage{placeins}
\usepackage{etoolbox}
\usepackage{tcolorbox}
\tcbuselibrary{breakable}

\newtheorem{theorem}{Theorem}[section]
\newtheorem{proposition}[theorem]{Proposition}
\newtheorem{lemma}[theorem]{Lemma}
\newtheorem{corollary}[theorem]{Corollary}
\theoremstyle{definition}
\newtheorem{definition}[theorem]{Definition}

\theoremstyle{remark}
\newtheorem{remark}[theorem]{Remark}

\definecolor{TheoremGray}{HTML}{F5F5F5}
\newtcolorbox{theoremblock}{
    breakable,
    colback=TheoremGray,
    colframe=TheoremGray,
    boxrule=0pt,
    sharp corners,
    left=0pt,
    right=0pt,
    top=0pt,
    bottom=0pt,
    boxsep=0pt,
    before skip=6pt,
    after skip=6pt
}
\BeforeBeginEnvironment{theorem}{\begin{theoremblock}}
\AfterEndEnvironment{theorem}{\end{theoremblock}}
\BeforeBeginEnvironment{proposition}{\begin{theoremblock}}
\AfterEndEnvironment{proposition}{\end{theoremblock}}
\BeforeBeginEnvironment{lemma}{\begin{theoremblock}}
\AfterEndEnvironment{lemma}{\end{theoremblock}}
\BeforeBeginEnvironment{corollary}{\begin{theoremblock}}
\AfterEndEnvironment{corollary}{\end{theoremblock}}
\BeforeBeginEnvironment{definition}{\begin{theoremblock}}
\AfterEndEnvironment{definition}{\end{theoremblock}}
\BeforeBeginEnvironment{example}{\begin{theoremblock}}
\AfterEndEnvironment{example}{\end{theoremblock}}
\BeforeBeginEnvironment{remark}{\begin{theoremblock}}
\AfterEndEnvironment{remark}{\end{theoremblock}}

\newcommand{\R}{\mathbb{R}}

\newcommand{\cL}{\mathcal{L}}
\newcommand{\cR}{\mathcal{R}}
\newcommand{\cT}{\mathcal{T}}

\newcommand{\dimB}{\dim_{\mathrm{B}}}
\newcommand{\abs}[1]{\lvert #1 \rvert}
\newcommand{\norm}[1]{\lVert #1 \rVert}

\hypersetup{
  colorlinks=true,
  linkcolor=blue!70!black,
  citecolor=green!50!black,
  urlcolor=blue!70!black
}

\makeatletter
\renewcommand{\maketitle}{%
  \begin{center}%
    {\Large\bfseries FI-KAN: Fractal Interpolation Kolmogorov--Arnold Networks\par}%
    \vspace{0.6em}%
    {\large \textbf{Gnankan Landry Regis N'guessan}$^{1,2,3}$\par}%
    \vspace{0.5em}%
    {\footnotesize
      $^1$Axiom Research Group\par
      \vspace{0.1em}%
      $^2$Dept.\ of Applied Mathematics and Computational Science, NM-AIST, Arusha, Tanzania\par
      \vspace{0.1em}%
      $^3$AIMS Research and Innovation Centre, Kigali, Rwanda\par
      \vspace{0.1em}%
      \texttt{rnguessan@aimsric.org}\par
    }%
  \end{center}%
  \vspace{-0.2em}%
}
\makeatother

\date{}

\begin{document}
\maketitle
\vspace{-1.0em}

{\small
\begin{abstract}
\noindent Kolmogorov--Arnold Networks (KAN) employ B-spline bases on a fixed grid, providing no intrinsic multi-scale decomposition for non-smooth function approximation. We introduce \textbf{Fractal Interpolation KAN} (FI-KAN), which incorporates learnable fractal interpolation function (FIF) bases from iterated function system (IFS) theory into KAN.
Two variants are presented: \textbf{Pure FI-KAN} (Barnsley, 1986) replaces B-splines entirely with FIF bases; \textbf{Hybrid FI-KAN} (Navascu\'es, 2005) retains the B-spline path and adds a learnable fractal correction.
The IFS contraction parameters give each edge a differentiable fractal dimension that adapts to target regularity during training.
On a H\"older regularity benchmark ($\alpha \in [0.2, 2.0]$), Hybrid FI-KAN outperforms KAN at every regularity level ($1.3\times$ to $\mathbf{33\times}$).
On fractal targets, FI-KAN achieves up to $\mathbf{6.3\times}$ MSE reduction over KAN, maintaining $4.7\times$ advantage at $5$~dB SNR.
On non-smooth PDE solutions (\texttt{scikit-fem}), Hybrid FI-KAN achieves up to $\mathbf{79\times}$ improvement on rough-coefficient diffusion and $3.5\times$ on L-shaped domain corner singularities.
Pure FI-KAN's complementary behavior, dominating on rough targets while underperforming on smooth ones, provides controlled evidence that basis geometry must match target regularity.
A fractal dimension regularizer provides interpretable complexity control whose learned values recover the true fractal dimension of each target.
These results establish regularity-matched basis design as a principled strategy for neural function approximation.
\end{abstract}
}

\vspace{0.2em}

\vspace{0.2em}
\noindent\begin{minipage}{\textwidth}
\centering
\includegraphics[width=0.78\textwidth]{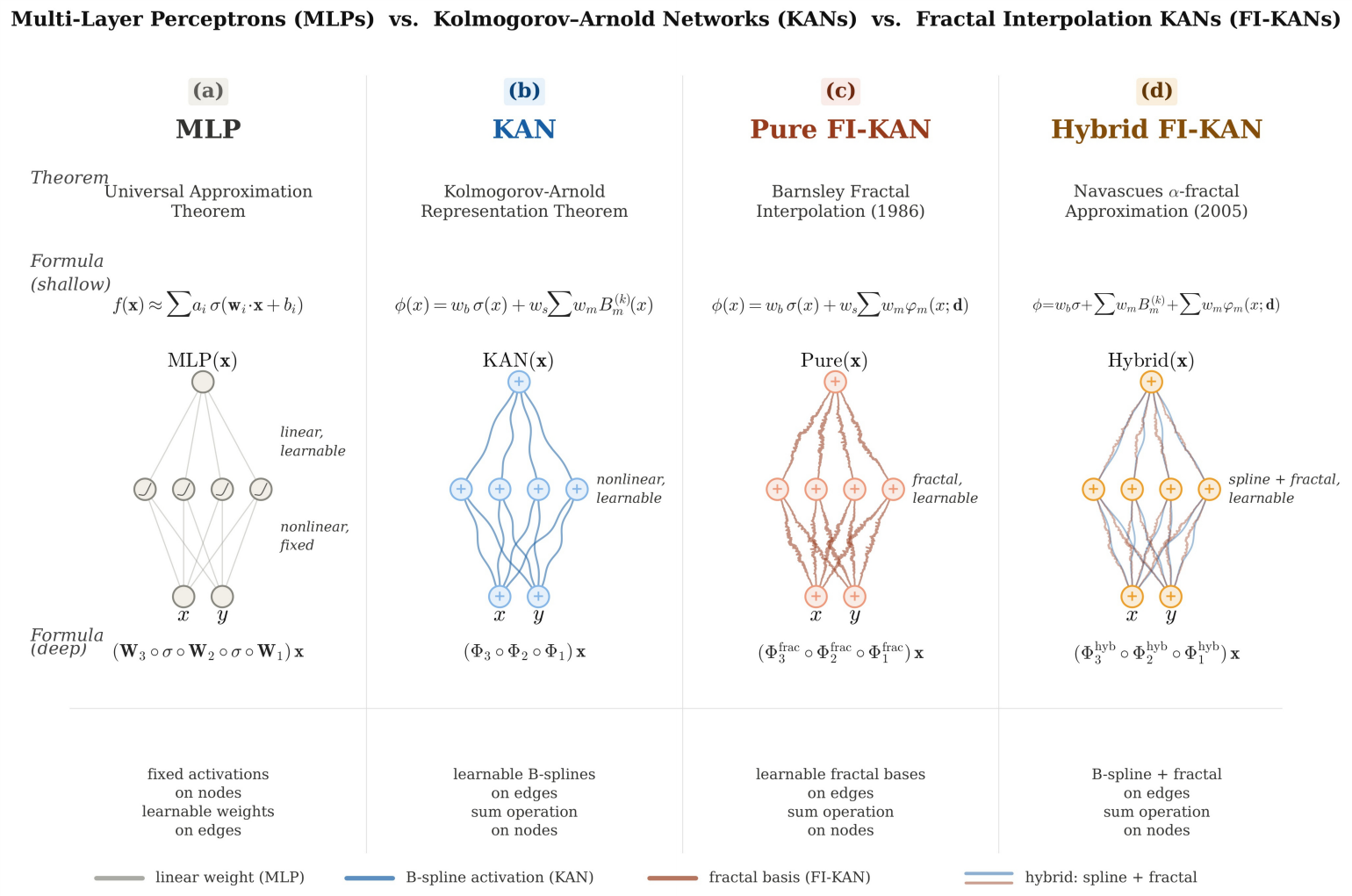}
\captionof{figure}{\small MLPs vs.\ KANs vs.\ FI-KANs. \textbf{(a)}~MLP: scalar weights on edges, fixed activations on nodes. \textbf{(b)}~KAN: learnable B-spline functions on edges. \textbf{(c)}~Pure FI-KAN: replaces B-splines with fractal interpolation bases $\varphi_m(x;\mathbf{d})$. \textbf{(d)}~Hybrid FI-KAN: retains B-splines and adds a fractal correction ($f_b^\alpha = b + h$). When $\mathbf{d} = \mathbf{0}$, (d) reduces to (b).}
\label{fig:architecture-comparison}
\end{minipage}

\newpage

\section{Introduction}
\label{sec:intro}

Neural function approximation architectures embed implicit assumptions about the regularity of their targets through their choice of basis functions.
Multi-layer perceptrons (MLPs) with smooth activation functions (ReLU, SiLU, GELU) construct approximants in spaces of piecewise smooth or analytic functions~\cite{cybenko1989approximation,hornik1989multilayer}.
Kolmogorov--Arnold Networks (KAN)~\cite{liu2024kan}, motivated by the Kolmogorov--Arnold representation theorem~\cite{kolmogorov1957representation,arnold1957functions}, replace fixed activations with learnable univariate functions parameterized as B-spline expansions~\cite{deboor2001practical}.
B-splines of order $k$ reproduce polynomials of degree at most $k-1$ and provide near-optimal approximation rates for targets in Sobolev and Besov spaces with integer or high fractional smoothness.

However, many functions of scientific and engineering interest are not smooth.
Turbulence velocity fields, financial time series, fracture surfaces, natural terrain profiles, and biomedical signals with multi-scale oscillations all exhibit non-trivial H\"older regularity, fractal self-similarity, or nowhere-differentiable character.
For such targets the smooth basis functions used by both MLPs and KANs are fundamentally mismatched: approximating a function with box-counting dimension $\dimB > 1$ using a smooth basis at resolution $h$ requires $O(h^{-1/\alpha})$ basis elements (where $\alpha$ is the H\"older exponent), with no gain from the polynomial reproduction properties that make splines efficient for smooth targets.

This paper introduces Fractal Interpolation KAN (FI-KAN), which augments or replaces the B-spline bases in KAN with fractal interpolation function (FIF) bases derived from iterated function system (IFS) theory~\cite{barnsley1986fractal,hutchinson1981fractals}.
The key innovation is that the vertical contraction parameters $\{d_i\}$ of the IFS are treated as trainable parameters, giving each edge activation a \emph{differentiable fractal dimension} that adapts to the regularity structure of the target function during training.

\paragraph{Contributions.}
\begin{enumerate}[leftmargin=2em,itemsep=2pt]
\item \textbf{Two architectures grounded in fractal approximation theory.}
  \emph{Pure FI-KAN} (Barnsley framework) replaces B-splines entirely with FIF bases.
  \emph{Hybrid FI-KAN} (Navascu\'es framework) retains B-splines and adds a fractal correction path.
  Both architectures are derived from classical mathematical frameworks, not ad hoc modifications.
\item \textbf{Learnable fractal dimension.}
  The contraction parameters provide a continuous, differentiable knob from smooth (piecewise linear, $\dimB = 1$) to rough (fractal, $\dimB > 1$) basis functions, learned from data.
\item \textbf{Fractal dimension regularization.}
  A differentiable regularizer penalizes unnecessary fractal complexity, implementing Occam's razor at the level of function geometry rather than parameter count.
\item \textbf{Comprehensive experimental validation.}
  Across functions spanning the H\"older regularity spectrum, FI-KAN demonstrates that matching basis regularity to target regularity yields substantial approximation gains on non-smooth targets, with additional advantages in noise robustness and continual learning.
\item \textbf{Empirical validation of the regularity-matching hypothesis.}
  The contrast between Pure and Hybrid FI-KAN provides controlled evidence that the geometric structure of the basis functions, not merely their number, is a critical design variable.
\item \textbf{Validation on non-smooth PDE solutions.}
  On reference solutions computed via \texttt{scikit-fem}~\cite{gustafsson2020scikit} for elliptic PDEs with corner singularities (H\"older $2/3$) and rough coefficients generated by fractional Brownian motion (\texttt{fbm} package~\cite{flynn2019fbm}), Hybrid FI-KAN achieves $65$--$79\times$ improvement over KAN, demonstrating that the regularity-matching advantage extends to structured roughness inherited from PDE operators.
\end{enumerate}

\paragraph{Scope.}
We do not claim FI-KAN as a general-purpose replacement for KAN.
We claim it as a principled extension for function classes with non-trivial geometric regularity, supported by both theory and experiment.
On smooth targets where B-splines are near-optimal, the Hybrid variant remains competitive (because the spline path carries the load), while the Pure variant underperforms (because its fractal bases cannot efficiently represent smooth curvature).
This asymmetry is not a weakness but a confirmation of the regularity-matching principle.
The advantage is most pronounced on targets with \emph{structured} roughness: PDE solutions inheriting non-smooth character from corner singularities, rough coefficients, or stochastic forcing, where Hybrid FI-KAN achieves up to $79\times$ improvement over KAN.

\paragraph{Organization.}
\Cref{sec:prelim} reviews the mathematical background on KAN, fractal interpolation functions, and $\alpha$-fractal approximation.
\Cref{sec:architecture} presents the FI-KAN architecture in both Pure and Hybrid variants.
\Cref{sec:theory} develops the approximation-theoretic analysis.
\Cref{sec:experiments} provides comprehensive experimental results.
\Cref{sec:related} discusses related work.
\Cref{sec:discussion} addresses limitations and future directions.
\Cref{sec:conclusion} concludes.

\begin{figure}[!htbp]
\centering
\includegraphics[width=\textwidth]{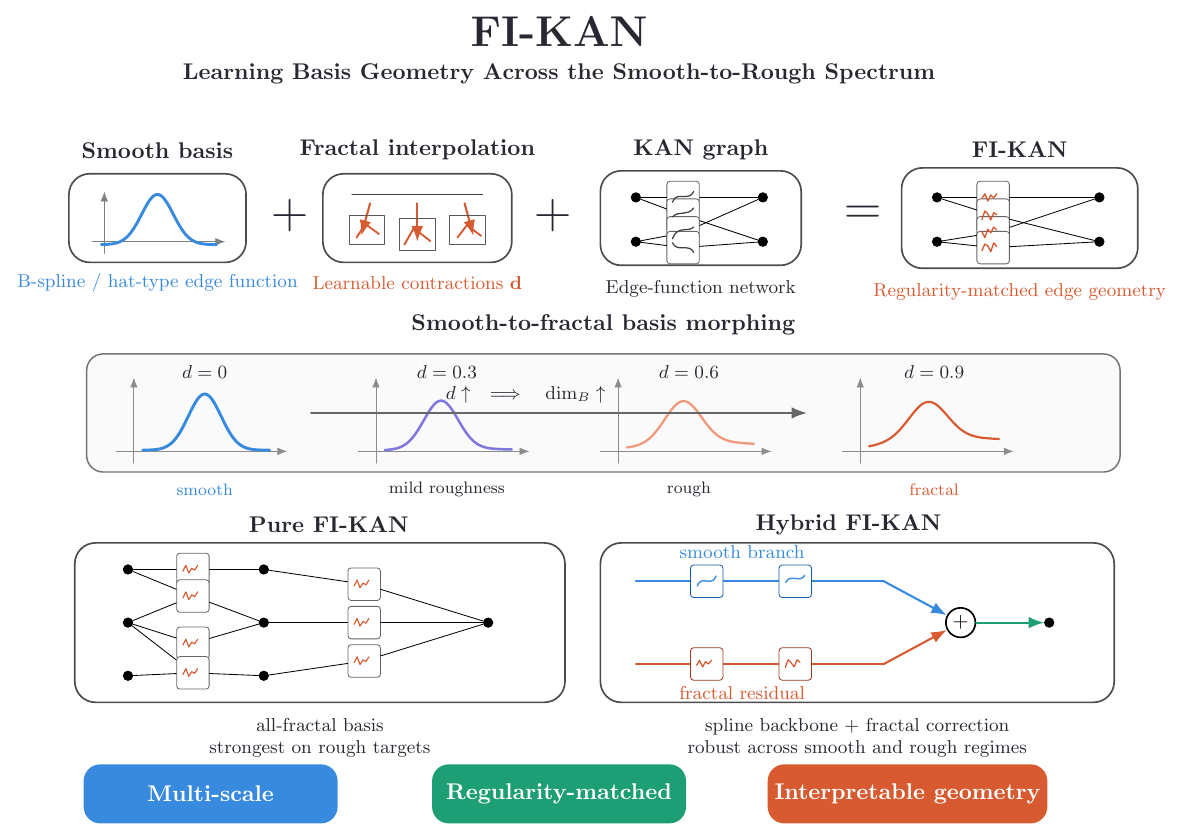}
\caption{FI-KAN: learning basis geometry across the smooth-to-rough spectrum.
\textbf{Top row:} FI-KAN combines smooth basis functions (B-spline/hat-type),
fractal interpolation with learnable contraction parameters $\mathbf{d}$,
and the KAN edge-function graph to produce regularity-matched edge geometry.
\textbf{Middle row:} Smooth-to-fractal basis morphing as $d_i$ increases from $0$ to $0.9$.
At $d_i = 0$ the basis is a smooth hat function ($\dim_B = 1$); as $d_i$ increases,
the basis acquires progressively finer self-affine structure with $\dim_B > 1$.
\textbf{Bottom row:} The two FI-KAN variants.
Pure FI-KAN (Barnsley framework) uses all-fractal bases, strongest on rough targets.
Hybrid FI-KAN (Navascu\'es framework) retains a spline backbone and adds a fractal correction,
providing robustness across both smooth and rough regimes.}
\label{fig:basis-geometry-spectrum}
\end{figure}

\section{Preliminaries}
\label{sec:prelim}

\subsection{Kolmogorov--Arnold Networks}
\label{sec:prelim:kan}

The Kolmogorov--Arnold representation theorem~\cite{kolmogorov1957representation,arnold1957functions} states that every continuous function $f\colon [0,1]^n \to \R$ admits a representation
\begin{equation}
  f(x_1, \ldots, x_n) = \sum_{q=0}^{2n} \Phi_q\!\left(\sum_{p=1}^n \psi_{q,p}(x_p)\right),
  \label{eq:ka-theorem}
\end{equation}
where $\Phi_q\colon \R \to \R$ and $\psi_{q,p}\colon [0,1] \to \R$ are continuous univariate functions.
KAN~\cite{liu2024kan} generalizes this by constructing neural networks whose edges carry learnable univariate functions rather than scalar weights.
Specifically, a KAN layer maps $\R^{n_{\mathrm{in}}} \to \R^{n_{\mathrm{out}}}$ via
\begin{equation}
  \mathbf{y} = \begin{pmatrix} \phi_{1,1}(\cdot) & \cdots & \phi_{1,n_{\mathrm{in}}}(\cdot) \\ \vdots & \ddots & \vdots \\ \phi_{n_{\mathrm{out}},1}(\cdot) & \cdots & \phi_{n_{\mathrm{out}},n_{\mathrm{in}}}(\cdot) \end{pmatrix} \mathbf{x},
  \label{eq:kan-layer}
\end{equation}
where each $\phi_{j,i}\colon \R \to \R$ is parameterized as a B-spline expansion plus a residual base activation.

In the efficient-KAN implementation~\cite{blealtan2024efficientkan}, each edge function takes the form
\begin{equation}
  \phi_{j,i}(x) = w^{(\mathrm{base})}_{j,i} \, \sigma(x) + w^{(\mathrm{scale})}_{j,i} \sum_{m=0}^{G+k-1} w^{(\mathrm{spline})}_{j,i,m} B_m^{(k)}(x),
  \label{eq:kan-edge}
\end{equation}
where $\sigma$ is SiLU, $\{B_m^{(k)}\}$ are B-spline basis functions of order $k$ on a grid of size $G$, and $w^{(\mathrm{base})}$, $w^{(\mathrm{scale})}$, $w^{(\mathrm{spline})}$ are learnable parameters.

\paragraph{B-spline basis properties.}
B-splines of order $k$ form a partition of unity, are $C^{k-2}$ smooth, and reproduce polynomials of degree at most $k-1$~\cite{deboor2001practical}.
For a target $f \in C^s([a,b])$ with $s \leq k$, the best $B$-spline approximation error on a uniform grid of spacing $h$ satisfies $\norm{f - f_h}_\infty = O(h^s)$.
This makes B-splines near-optimal for smooth targets but provides no structural advantage when $s < 1$ (i.e., H\"older-continuous but not Lipschitz) or when the target has fractal character.

\subsection{Fractal Interpolation Functions}
\label{sec:prelim:fif}

Fractal interpolation, introduced by Barnsley~\cite{barnsley1986fractal}, constructs continuous functions whose graphs can have prescribed fractal dimension.
The construction uses the theory of iterated function systems (IFS)~\cite{hutchinson1981fractals,barnsley1988fractals}.

\begin{definition}[Fractal Interpolation Function]
\label{def:fif}
Given interpolation data $\{(x_i, y_i)\}_{i=0}^N$ with $a = x_0 < x_1 < \cdots < x_N = b$, consider the IFS $\{w_i\}_{i=1}^N$ on $[a,b] \times \R$ defined by
\begin{equation}
  w_i(x, y) = \begin{pmatrix} a_i & 0 \\ c_i & d_i \end{pmatrix} \begin{pmatrix} x \\ y \end{pmatrix} + \begin{pmatrix} e_i \\ f_i \end{pmatrix}, \quad i = 1, \ldots, N,
  \label{eq:ifs}
\end{equation}
where:
\begin{itemize}[leftmargin=2em,itemsep=2pt]
  \item $\abs{d_i} < 1$ for all $i$ (vertical contractivity);
  \item $a_i, e_i$ are determined by the interpolation constraints:
    $w_i(x_0, y_0) = (x_{i\text{-}1}, y_{i\text{-}1})$,
    $w_i(x_N, y_N) = (x_i, y_i)$;
  \item $d_i \in (-1,1)$ are the \emph{vertical scaling} (contraction) factors, the only free parameters of the IFS;
  \item $c_i$ and $f_i$ are determined by the interpolation constraints and the choice of~$d_i$:
    \begin{equation}
      c_i = \frac{y_i - y_{i-1} - d_i(y_N - y_0)}{x_N - x_0}, \qquad f_i = y_{i-1} - c_i x_0 - d_i y_0.
      \label{eq:ci-fi-determined}
    \end{equation}
\end{itemize}
The \emph{fractal interpolation function} (FIF) $f^*$ is the unique continuous function whose graph $G(f^*)$ is the attractor of this IFS.
\end{definition}

The maps $L_i(x) = a_i x + e_i$ project $[x_0, x_N]$ onto $[x_{i-1}, x_i]$.
The FIF $f^*$ satisfies the Read--Bajraktarevi\'c (RB) functional equation~\cite{barnsley1986fractal}:
\begin{equation}
  f^*(L_i(x)) = c_i \, x + d_i \, f^*(x) + f_i, \quad x \in [x_0, x_N], \quad i = 1, \ldots, N.
  \label{eq:rb-equation}
\end{equation}

\begin{theorem}[Barnsley, 1986~\cite{barnsley1986fractal}]
\label{thm:barnsley-existence}
If $\abs{d_i} < 1$ for all $i = 1, \ldots, N$, then the IFS~\eqref{eq:ifs} has a unique attractor that is the graph of a continuous function $f^*\colon [a,b] \to \R$ satisfying $f^*(x_i) = y_i$ for all $i$.
\end{theorem}

\begin{theorem}[Fractal Dimension; Barnsley, 1986~\cite{barnsley1986fractal}]
\label{thm:fractal-dim}
For the FIF $f^*$ with $c_i = 0$ and $\abs{d_i} < 1$, the box-counting dimension of the graph satisfies
\begin{equation}
  \dimB\bigl(\mathrm{Graph}(f^*)\bigr) =
  \begin{cases}
    1 & \text{if } \displaystyle\sum_{i=1}^N \abs{d_i} \leq 1, \\[8pt]
    1 + \dfrac{\log \sum_{i=1}^N \abs{d_i}}{\log N} & \text{if } \displaystyle\sum_{i=1}^N \abs{d_i} > 1.
  \end{cases}
  \label{eq:fractal-dim}
\end{equation}
\end{theorem}

\paragraph{Linearity in the ordinates.}
A key structural property is that the FIF depends \emph{linearly} on the interpolation ordinates: $f^*(x; \mathbf{y}, \mathbf{d}) = \sum_{i=0}^N y_i \, \varphi_i(x; \mathbf{d})$, where the \emph{FIF basis functions} $\varphi_i$ satisfy $\varphi_i(x_j; \mathbf{d}) = \delta_{ij}$ (the Kronecker property).
This parallels the structure of B-spline expansions and is crucial for embedding FIF bases in the KAN framework.

\paragraph{The $c_i = 0$ specialization.}
When $c_i = 0$ for all $i$, the RB equation~\eqref{eq:rb-equation} simplifies to
\begin{equation}
  f^*(L_i(x)) = d_i \, f^*(x) + f_i, \quad x \in [x_0, x_N].
  \label{eq:rb-simple}
\end{equation}
This is the standard ``recurrent'' FIF studied extensively in fractal approximation theory~\cite{barnsley1986fractal,massopust2010interpolation,massopust1994fractal}.
The basis functions $\varphi_i(x; \mathbf{d})$ depend only on $\mathbf{d}$ and the grid structure.
\Cref{thm:fractal-dim} applies directly.
When $d_i = 0$ for all $i$, the FIF reduces to the piecewise linear interpolant through $\{(x_i, y_i)\}$, i.e., the basis functions become the standard hat functions.

\begin{remark}[The $c_i = 0$ specialization and basis function construction]
\label{rem:ci-zero}
Setting $c_i = 0$ in~\eqref{eq:ci-fi-determined} imposes $d_i = (y_i - y_{i-1})/(y_N - y_0)$ when $y_0 \neq y_N$, constraining $d_i$ rather than leaving it free.
For the FIF basis functions $\varphi_j$ with Kronecker data $(x_i, \delta_{ij})$, the interior bases ($0 < j < N$) have $y_0 = y_N = 0$, so the constraint degenerates and $d_i$ remains free, but the endpoint conditions yield non-trivial $c_i$ values for subintervals adjacent to the $j$-th grid point.
\Cref{alg:fractal-bases} handles this correctly: the piecewise linear base case and boundary corrections implicitly encode the $c_i$ contributions determined by the Kronecker data.
The fractal dimension formula (\cref{thm:fractal-dim}) remains valid because $c_i$ affects only the linear skeleton of the FIF, not the self-affine scaling structure that governs the box-counting dimension~\cite{barnsley1986fractal,falconer2003fractal}.
Thus ``$c_i = 0$'' should be understood as a simplifying label for the analysis, not a literal constraint on the implementation.
\end{remark}

\subsection{Alpha-Fractal Interpolation}
\label{sec:prelim:alpha}

Navascu\'es~\cite{navascues2005fractal} generalized Barnsley's construction by introducing the $\alpha$-fractal operator.
Given a base function $b \in C([a,b])$ that interpolates the data $\{(x_i, y_i)\}_{i=0}^N$, the \emph{$\alpha$-fractal function} $f_b^\alpha$ is defined as the FIF satisfying
\begin{equation}
  f_b^\alpha(L_i(x)) = \alpha_i \, f_b^\alpha(x) + b(L_i(x)) - \alpha_i \, b(x),
  \label{eq:navascues}
\end{equation}
where $\alpha_i \in (-1,1)$ are the fractal parameters and $b$ is the \emph{base function} (a classical, typically smooth, approximant).

Setting $h = f_b^\alpha - b$, one obtains
\begin{equation}
  h(L_i(x)) = \alpha_i \, h(x), \quad h(x_j) = 0 \text{ for all } j,
  \label{eq:navascues-perturbation}
\end{equation}
so $h$ is a self-affine perturbation that vanishes at the interpolation points.
The $\alpha$-fractal function thus decomposes as
\begin{equation}
  f_b^\alpha = \underbrace{b}_{\text{classical approximant}} + \underbrace{h}_{\text{fractal perturbation}}.
  \label{eq:alpha-decomposition}
\end{equation}

This framework has three key properties:
\begin{enumerate}[leftmargin=2em,itemsep=2pt]
  \item \textbf{Recovery:} When $\alpha_i = 0$ for all $i$, $f_b^\alpha = b$ (the base function is recovered exactly).
  \item \textbf{Continuous bridge:} The parameters $\alpha_i$ provide a continuous transition from classical to fractal approximation.
  \item \textbf{Residual structure:} The perturbation $h$ captures precisely what the smooth base function $b$ misses. If $b$ is a spline and the target has fractal structure, then $h$ encodes the non-smooth residual.
\end{enumerate}

This decomposition provides the theoretical foundation for our Hybrid FI-KAN architecture.

\section{FI-KAN Architecture}
\label{sec:architecture}

We present two variants of Fractal Interpolation KAN, each grounded in one of the mathematical frameworks described in \cref{sec:prelim}.

\subsection{FIF Basis Function Computation}
\label{sec:arch:bases}

\begin{figure}[!htbp]
\centering
\includegraphics[width=\textwidth]{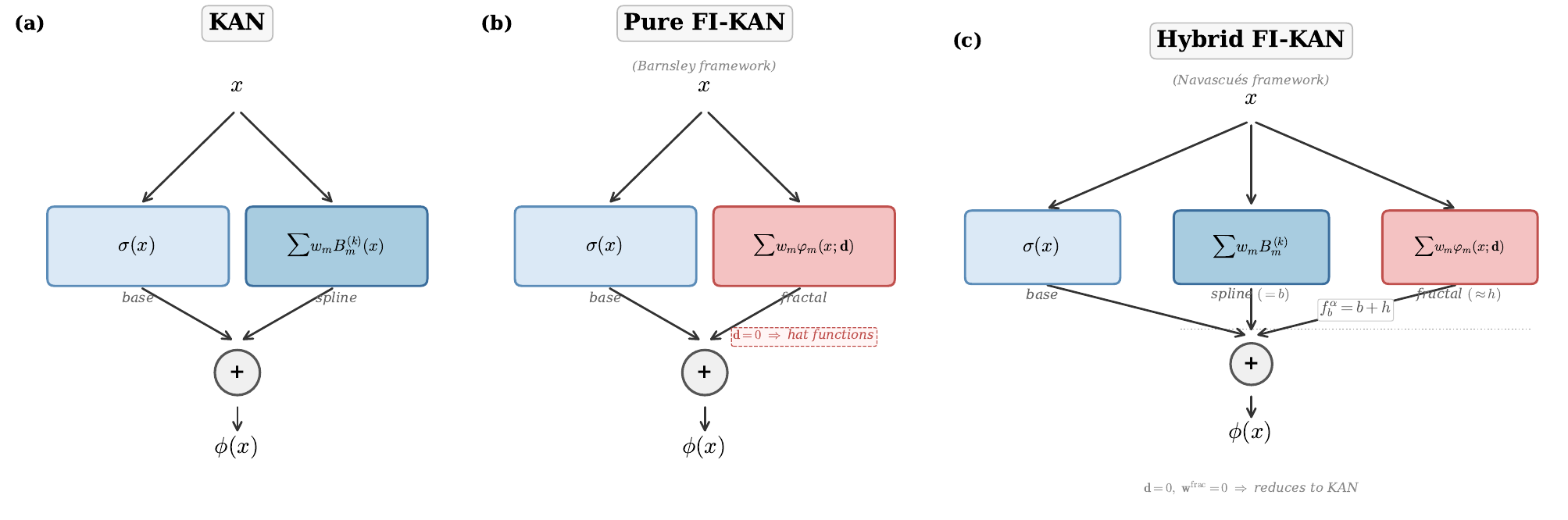}
\caption{Edge function architecture for the three models. \textbf{(a) KAN}: base activation plus B-spline path. \textbf{(b) Pure FI-KAN} (Barnsley): replaces B-splines with fractal interpolation bases $\varphi_m(x; \mathbf{d})$. When $\mathbf{d} = \mathbf{0}$, the FIF bases reduce to piecewise linear hat functions. \textbf{(c) Hybrid FI-KAN} (Navascu\'es): retains the B-spline path and adds a fractal correction, implementing the $\alpha$-fractal decomposition $f_b^\alpha = b + h$. When $\mathbf{d} = \mathbf{0}$ and fractal weights vanish, Hybrid reduces to standard KAN.}
\label{fig:architecture}
\end{figure}

We construct FIF basis functions $\{\varphi_i(x; \mathbf{d})\}_{i=0}^N$ on a uniform grid $x_i = a + i(b-a)/N$ for $i = 0, \ldots, N$, using the $c_i = 0$ specialization of Barnsley's framework (\cref{sec:prelim:fif}).
Evaluation uses a truncated iteration of the RB operator to depth $K$.

\begin{algorithm}[t]
\caption{Fractal Basis Function Evaluation}
\label{alg:fractal-bases}
\KwIn{$\mathbf{x} \in [a,b]^{B \times D}$ (batch $\times$ features), $\mathbf{d} \in (-1,1)^{D \times N}$ (contraction parameters), grid size $N$, recursion depth $K$}
\KwOut{$\boldsymbol{\Phi} \in \R^{B \times D \times (N+1)}$ (basis function values)}
$u \gets (x - a) / (b - a)$, clamped to $(\varepsilon, 1 - \varepsilon)$ \tcp*{Normalize to $[0,1]$}
$\boldsymbol{\Phi} \gets \mathbf{0}$, \quad $r \gets \mathbf{1}$ \tcp*{Initialize bases and running product}
\For{$k = 0, \ldots, K-1$}{
    $j \gets \lfloor u \cdot N \rfloor$, clamped to $\{0, \ldots, N-1\}$ \tcp*{Interval index}
    $t \gets u \cdot N - j$ \tcp*{Local coordinate}
    $d_j \gets \mathbf{d}[\cdot, j]$ \tcp*{Gather contraction factor}
    $\boldsymbol{\Phi}[\cdot, \cdot, j] \mathrel{+}= r \cdot (1 - t)$;
    \quad $\boldsymbol{\Phi}[\cdot, \cdot, j+1] \mathrel{+}= r \cdot t$ \\
    $\boldsymbol{\Phi}[\cdot, \cdot, 0] \mathrel{-}= r \cdot d_j \cdot (1-t)$;
    \quad $\boldsymbol{\Phi}[\cdot, \cdot, N] \mathrel{-}= r \cdot d_j \cdot t$ \tcp*{Boundary corrections}
    $r \gets r \cdot d_j$; \quad $u \gets t$
}
\tcp{Base case: piecewise linear}
$j \gets \lfloor u \cdot N \rfloor$; \quad $t \gets u \cdot N - j$ \\
$\boldsymbol{\Phi}[\cdot, \cdot, j] \mathrel{+}= r \cdot (1 - t)$; \quad $\boldsymbol{\Phi}[\cdot, \cdot, j+1] \mathrel{+}= r \cdot t$
\end{algorithm}

The boundary corrections at indices $0$ and $N$ enforce the endpoint constraints of the RB operator. The running product $r_k = \prod_{m=0}^{k-1} d_{\sigma_m}$ tracks the cumulative contraction through $k$ levels of recursion, where $\sigma_m$ is the interval index at recursion depth $m$.

\begin{proposition}[Truncation Error]
\label{prop:truncation}
Let $d_{\max} = \max_{i} \abs{d_i} < 1$. The truncated evaluation at depth $K$ satisfies
\begin{equation}
  \norm{f^{(K)} - f^*}_\infty \leq C \cdot d_{\max}^K,
\end{equation}
where $C$ depends on the interpolation ordinates and the grid.
\end{proposition}

\begin{proof}
At depth $K$ the running product satisfies $\abs{r_K} = \prod_{m=0}^{K-1} \abs{d_{\sigma_m}} \leq d_{\max}^K$. The remaining contribution to the FIF from recursion depths $K, K+1, \ldots$ is bounded by $C \sum_{j=K}^\infty d_{\max}^j = C \, d_{\max}^K / (1 - d_{\max})$, which is exponentially small in $K$.
\end{proof}

\paragraph{Differentiability.}
Each term in the truncated expansion is polynomial in $\{d_i, y_i\}$ (specifically, products and sums of the contraction parameters and ordinates with piecewise polynomial functions of $x$).
The entire forward pass is therefore differentiable with respect to all parameters and compatible with automatic differentiation through PyTorch~\cite{paszke2019pytorch}.

\paragraph{Contraction parameter reparameterization.}
To enforce $\abs{d_i} < 1$ while maintaining unconstrained optimization, we parameterize $d_i = d_{\max} \cdot \tanh(d_i^{(\mathrm{raw})})$ where $d_i^{(\mathrm{raw})} \in \R$ is the unconstrained learnable parameter and $d_{\max} = 0.99$.

\begin{figure}[!htbp]
\centering
\includegraphics[width=0.92\textwidth]{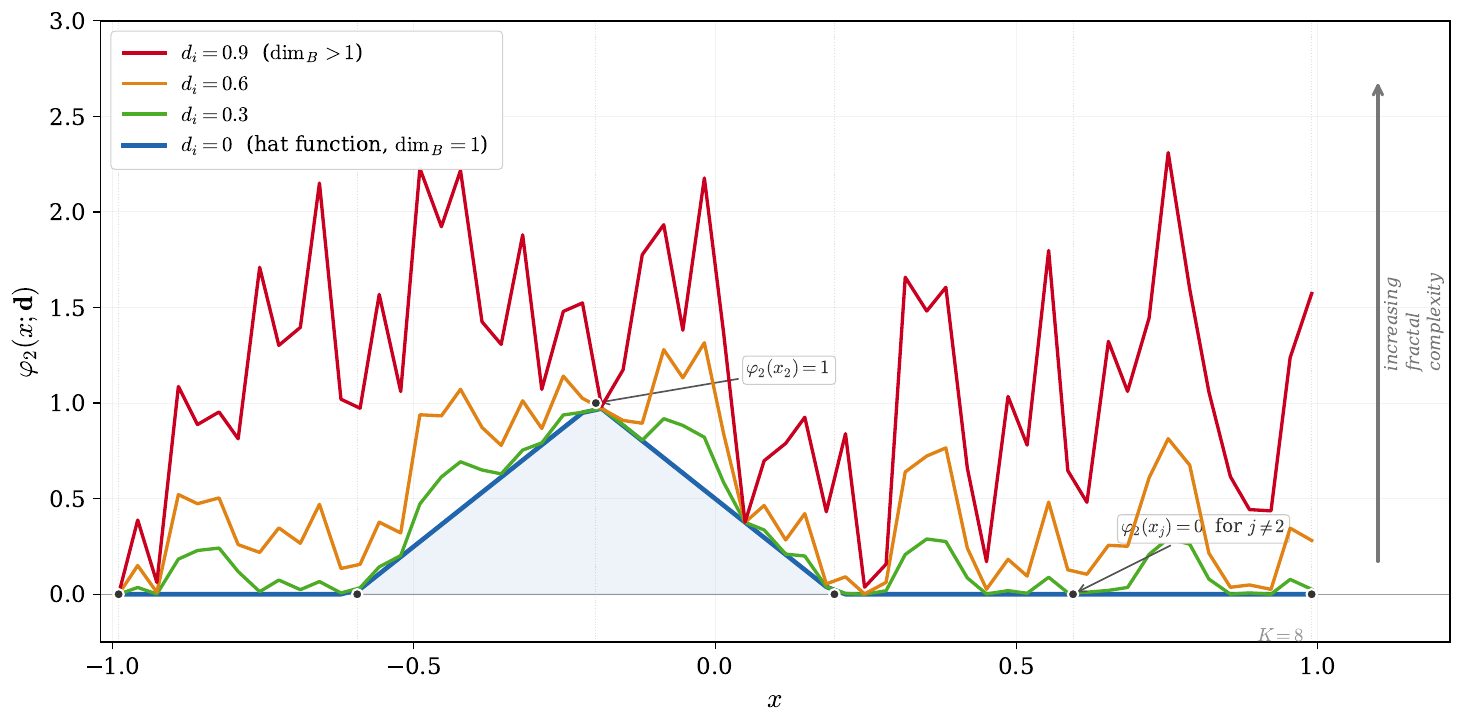}
\caption{Fractal basis function $\varphi_2(x; \mathbf{d})$ (the middle basis on a 5-interval grid) as the contraction parameters $d_i$ vary uniformly from 0 to 0.9. At $d_i = 0$, the basis is the standard piecewise linear hat function ($\dimB = 1$). As $d_i$ increases, the basis acquires increasingly fine-scale fractal structure ($\dimB > 1$) while maintaining the Kronecker property $\varphi_2(x_j) = \delta_{2j}$ at all grid points (black dots). Recursion depth $K = 8$.}
\label{fig:basis-morphing}
\end{figure}

\subsection{Pure FI-KAN (Barnsley Framework)}
\label{sec:arch:pure}

Pure FI-KAN replaces B-spline basis functions entirely with FIF bases, testing the regularity-matching hypothesis in its strongest form.

\begin{figure}[!htbp]
\centering
\includegraphics[width=\textwidth]{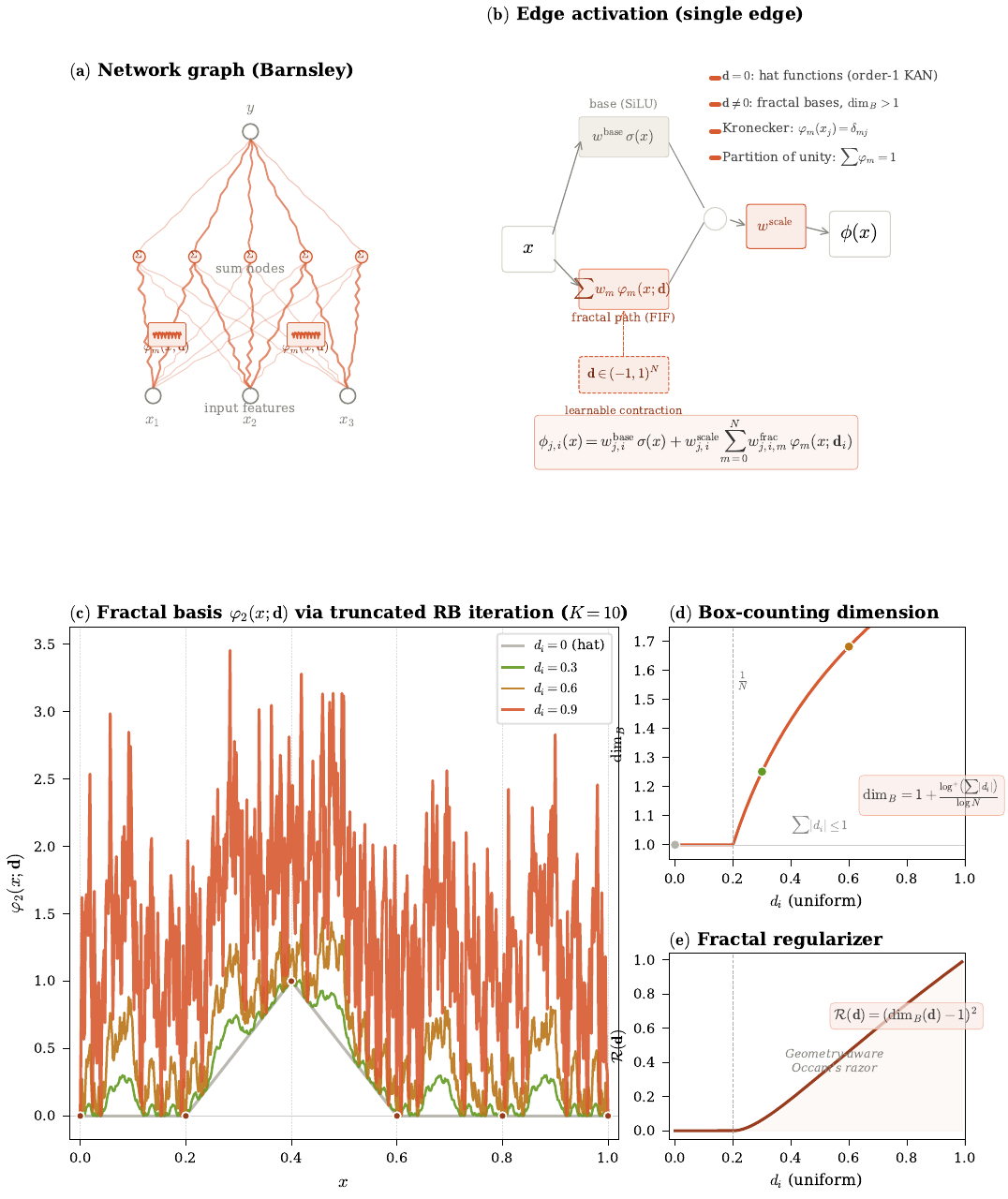}
\caption{Pure FI-KAN (Barnsley framework) architecture detail.
\textbf{(a)}~Network graph: all edges carry fractal interpolation function (FIF) bases
$\varphi_m(x;\mathbf{d})$ with learnable contraction parameters; nodes perform summation.
\textbf{(b)}~Edge activation: the input $x$ splits into a base SiLU path and a fractal path
$\sum w_m^{\mathrm{frac}} \varphi_m(x;\mathbf{d})$, where $\mathbf{d}\in(-1,1)^N$ controls the
fractal character of the basis.
When $\mathbf{d}=\mathbf{0}$, the FIF bases reduce to piecewise linear hat functions (order-1 KAN).
\textbf{(c)}~Fractal basis $\varphi_2(x;\mathbf{d})$ at four contraction values
$d_i \in \{0, 0.3, 0.6, 0.9\}$, computed via $K=10$ truncated Read--Bajraktarevi\'c iterations.
The Kronecker property $\varphi_2(x_j) = \delta_{2j}$ is preserved at all $d$ values (dots).
\textbf{(d)}~Box-counting dimension $\dim_B(\mathbf{d})$ as a function of uniform $d_i$,
with the transition surface $\sum|d_i| = 1$ marked.
\textbf{(e)}~Fractal dimension regularizer $\mathcal{R}(\mathbf{d}) = (\dim_B(\mathbf{d}) - 1)^2$:
a geometry-aware Occam's razor that penalizes unnecessary fractal complexity.}
\label{fig:pure-fikan-detail}
\end{figure}

\begin{definition}[Pure FI-KAN Edge Function]
Each edge activation in a Pure FI-KAN layer computes
\begin{equation}
  \phi_{j,i}(x) = w_{j,i}^{(\mathrm{base})} \, \sigma(x) + w_{j,i}^{(\mathrm{scale})} \sum_{m=0}^{N} w_{j,i,m}^{(\mathrm{frac})} \, \varphi_m(x; \mathbf{d}_i),
  \label{eq:pure-edge}
\end{equation}
where $\sigma$ is SiLU, $\{\varphi_m(\cdot; \mathbf{d}_i)\}_{m=0}^N$ are the FIF basis functions (\cref{alg:fractal-bases}), $w^{(\mathrm{frac})}_{j,i,m}$ are learnable interpolation ordinates (analogous to B-spline coefficients), and $\mathbf{d}_i \in (-1,1)^N$ are the learnable contraction parameters for input feature $i$.
\end{definition}

\paragraph{Learnable parameters per edge.}
\begin{itemize}[leftmargin=2em,itemsep=2pt]
  \item Interpolation ordinates: $\{w^{(\mathrm{frac})}_{j,i,m}\}_{m=0}^N$ ($N+1$ parameters, playing the role of B-spline coefficients).
  \item Contraction factors: $\{d_{i,m}\}_{m=1}^N$ ($N$ parameters, shared across output features, controlling fractal character).
  \item Base weight: $w^{(\mathrm{base})}_{j,i}$ and scale: $w^{(\mathrm{scale})}_{j,i}$ (2 parameters, as in standard KAN).
\end{itemize}

\paragraph{Properties.}
\begin{enumerate}[leftmargin=2em,itemsep=2pt]
  \item When $\mathbf{d} = \mathbf{0}$: all FIF bases reduce to piecewise linear hat functions. The edge functions become piecewise linear (order-1 spline) KAN edges.
  \item When $\mathbf{d} \neq \mathbf{0}$: the basis functions acquire fractal structure. The box-counting dimension of each edge activation is a differentiable function of $\mathbf{d}$ via~\eqref{eq:fractal-dim}.
  \item The network learns \emph{both} the interpolation ordinates (what values to hit at grid points) \emph{and} the inter-grid-point geometry (how to interpolate between grid points, with learnable roughness).
\end{enumerate}

\paragraph{Inductive bias.}
Pure FI-KAN is biased toward targets with non-trivial fractal structure.
For smooth targets, it must learn $\mathbf{d} \approx \mathbf{0}$ to recover piecewise linear (hat function) bases, which form a weaker approximation class than B-splines of order $k \geq 2$ since they cannot reproduce quadratic or higher-degree polynomials.
This limitation is by design: it enables testing whether fractal bases are necessary and sufficient for rough targets without confounding from smooth-target performance.

\subsection{Hybrid FI-KAN (Navascu\'es Framework)}
\label{sec:arch:hybrid}

Hybrid FI-KAN instantiates Navascu\'es's $\alpha$-fractal decomposition (\cref{sec:prelim:alpha}) within the KAN architecture: the B-spline path serves as the classical approximant $b$ and a parallel FIF path provides the fractal correction $h$.

\begin{figure}[!htbp]
\centering
\includegraphics[width=\textwidth]{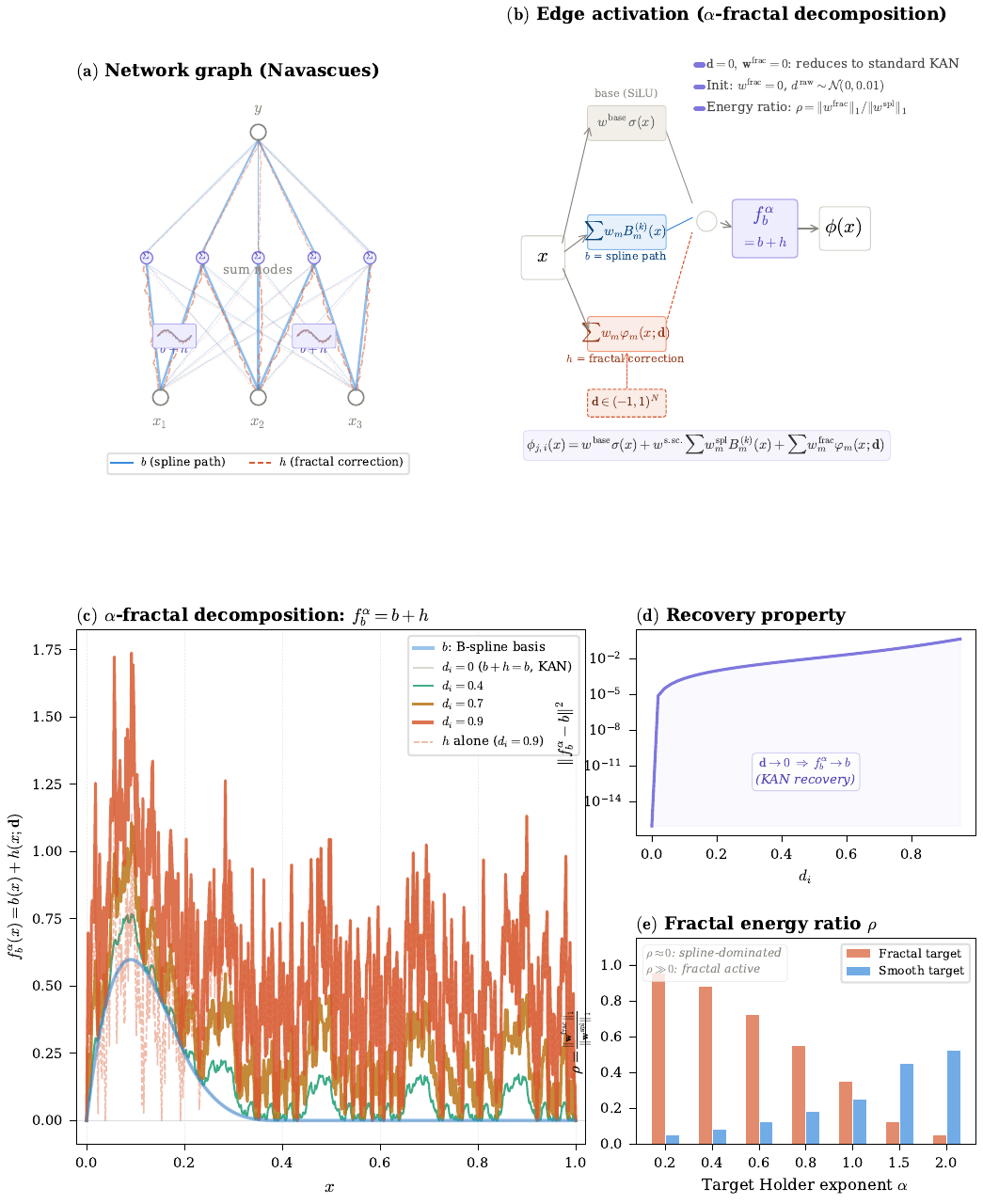}
\caption{Hybrid FI-KAN (Navascu\'es framework) architecture detail.
\textbf{(a)}~Network graph: edges carry dual paths---solid blue for the B-spline path $b$
(classical approximant) and dashed coral for the fractal correction $h$---implementing
the $\alpha$-fractal decomposition $f_b^\alpha = b + h$.
\textbf{(b)}~Edge activation: the input splits into three paths (base SiLU, B-spline, FIF),
combined into the $\alpha$-fractal output $f_b^\alpha = b + h$.
When $\mathbf{d}=\mathbf{0}$ and $\mathbf{w}^{\mathrm{frac}}=\mathbf{0}$,
the architecture reduces exactly to standard KAN.
Initialization sets $w^{\mathrm{frac}}=0$ and $d^{\mathrm{raw}}\sim\mathcal{N}(0,0.01)$,
so the network starts as a KAN and develops fractal structure only where the data demands it.
\textbf{(c)}~$\alpha$-fractal decomposition showing the B-spline basis $b$ (blue),
hybrid outputs $b+h$ at increasing $d_i$, and the isolated fractal correction $h$ (dashed).
\textbf{(d)}~Recovery property: $\|\mathbf{d}\|\to 0$ implies $f_b^\alpha \to b$ (KAN recovery).
\textbf{(e)}~Fractal energy ratio $\rho = \|w^{\mathrm{frac}}\|_1 / \|w^{\mathrm{spl}}\|_1$
across target H\"older exponents: fractal path is active ($\rho \gg 0$) on rough targets
and dormant ($\rho \approx 0$) on smooth targets.}
\label{fig:hybrid-fikan-detail}
\end{figure}

\begin{definition}[Hybrid FI-KAN Edge Function]
Each edge activation computes
\begin{equation}
  \phi_{j,i}(x) = \underbrace{w_{j,i}^{(\mathrm{base})} \sigma(x)}_{\text{base}} + \underbrace{w_{j,i}^{(\mathrm{s.sc.})} \sum_m w_{j,i,m}^{(\mathrm{spl})} B_m^{(k)}(x)}_{\text{spline path}~(= b)} + \underbrace{\sum_{m=0}^{N} w_{j,i,m}^{(\mathrm{frac})} \, \varphi_m(x; \mathbf{d}_i)}_{\text{fractal correction}~(\approx h)}.
  \label{eq:hybrid-edge}
\end{equation}
\end{definition}

This is the $\alpha$-fractal decomposition~\eqref{eq:alpha-decomposition} realized as a neural architecture:
\begin{equation}
  f_b^\alpha(x) \;=\; \underbrace{b(x)}_{\text{spline path}} \;+\; \underbrace{h(x; \mathbf{d})}_{\text{fractal correction path}}.
\end{equation}

\paragraph{Connection to Navascu\'es's framework.}
\begin{itemize}[leftmargin=2em,itemsep=2pt]
  \item The B-spline path $b(x) = \sum_m w_m^{(\mathrm{spl})} B_m^{(k)}(x)$ is the classical approximant.
  \item The fractal path $h(x; \mathbf{d}) = \sum_m w_m^{(\mathrm{frac})} \varphi_m(x; \mathbf{d})$ is the fractal perturbation.
  \item When $\mathbf{d} = \mathbf{0}$ and $\mathbf{w}^{(\mathrm{frac})} = \mathbf{0}$: the architecture reduces exactly to standard KAN, recovering Navascu\'es's property that $\alpha = 0$ gives back the base function.
  \item The implementation generalizes the strict Navascu\'es construction by allowing the fractal correction to have \emph{independent} interpolation ordinates, rather than constraining $h(x_j) = 0$.
    This gives the network more freedom to learn the optimal smooth-rough decomposition from data.
\end{itemize}

\paragraph{Initialization.}
The fractal weights $w^{(\mathrm{frac})}$ are initialized to zero and the contraction parameters $d^{(\mathrm{raw})}$ are initialized near zero ($d^{(\mathrm{raw})} \sim \mathcal{N}(0, 0.01)$).
The network therefore starts as a standard KAN and develops fractal structure only where the data demands it.
This provides a strong inductive bias toward the simplest explanation: use smooth splines unless fractal correction demonstrably reduces the loss.

\paragraph{Fractal energy ratio.}
The diagnostic quantity
\begin{equation}
  \rho = \frac{\norm{\mathbf{w}^{(\mathrm{frac})}}_1}{\norm{\mathbf{w}^{(\mathrm{spl})}}_1 + \varepsilon}
  \label{eq:fractal-energy}
\end{equation}
measures the magnitude of the fractal correction relative to the spline path.
When $\rho \approx 0$, the fractal path is inactive (spline-dominated behavior).
When $\rho > 0$, the network has learned to apply non-trivial fractal correction.

\subsection{Fractal Dimension Regularization}
\label{sec:arch:reg}

Both variants support a geometry-aware regularizer derived from IFS theory.

\begin{definition}[Fractal Dimension Regularizer]
\label{def:reg}
For an edge with contraction parameters $\mathbf{d} = (d_1, \ldots, d_N)$, define
\begin{equation}
  R_{\mathrm{edge}}(\mathbf{d}) = \bigl(\dimB(\mathbf{d}) - 1\bigr)^2, \qquad \dimB(\mathbf{d}) = 1 + \frac{\log^+\!\bigl(\sum_{i=1}^N \abs{d_i}\bigr)}{\log N},
  \label{eq:reg-edge}
\end{equation}
where $\log^+(x) = \max(\log x, 0)$.
The total fractal regularization loss is
\begin{equation}
  \cR_{\mathrm{fractal}}(\theta) = \sum_{\ell} \sum_{i} R_{\mathrm{edge}}(\mathbf{d}_i^{(\ell)}),
  \label{eq:reg-total}
\end{equation}
summed over all layers $\ell$ and input features $i$.
\end{definition}

\paragraph{Properties.}
\begin{enumerate}[leftmargin=2em,itemsep=2pt]
  \item \textbf{Geometry-aware Occam's razor.}
    The regularizer penalizes fractal dimension exceeding $1$, i.e., penalizes fractal structure in the basis functions. It says: prefer smooth bases unless the data provides sufficient evidence for fractal structure.
  \item \textbf{Differentiability.}
    The regularizer is differentiable with respect to $\mathbf{d}$ through the $\tanh$ reparameterization.
  \item \textbf{Interpretability.}
    The learned fractal dimension $\dimB(\mathbf{d})$ is a meaningful diagnostic: it tracks the geometric regularity of the target function (see \cref{sec:exp:diagnostic}).
  \item \textbf{Distinct from weight regularization.}
    This is \emph{not} $L_1$/$L_2$ regularization on parameter magnitude. It controls the \emph{geometric complexity} of the basis functions, which is a fundamentally different notion of model complexity.
\end{enumerate}

The total training loss is:
\begin{equation}
  \cL = \cL_{\mathrm{data}} + \lambda_{\mathrm{act}} \, \cR_{\mathrm{act}} + \lambda_{\mathrm{ent}} \, \cR_{\mathrm{ent}} + \lambda_{\mathrm{frac}} \, \cR_{\mathrm{fractal}},
  \label{eq:total-loss}
\end{equation}
where $\cR_{\mathrm{act}}$ and $\cR_{\mathrm{ent}}$ are the activation and entropy regularizers from KAN~\cite{liu2024kan}.

\subsection{Computational Considerations}
\label{sec:arch:compute}

\paragraph{Parameter count.}
For a layer with $n_{\mathrm{in}}$ inputs, $n_{\mathrm{out}}$ outputs, and grid size $G$:
\begin{itemize}[leftmargin=2em,itemsep=2pt]
  \item KAN: $n_{\mathrm{out}} \times n_{\mathrm{in}} \times (G + k)$ spline weights + $n_{\mathrm{out}} \times n_{\mathrm{in}}$ base and scale weights.
  \item Pure FI-KAN: $n_{\mathrm{out}} \times n_{\mathrm{in}} \times (G + 1)$ fractal weights + $n_{\mathrm{in}} \times G$ contraction parameters + $n_{\mathrm{out}} \times n_{\mathrm{in}}$ base and scale weights.
  \item Hybrid FI-KAN: KAN parameters + $n_{\mathrm{out}} \times n_{\mathrm{in}} \times (G+1)$ fractal weights + $n_{\mathrm{in}} \times G$ contraction parameters.
\end{itemize}

\paragraph{Computational cost.}
The FIF basis evaluation (\cref{alg:fractal-bases}) requires $K$ sequential iterations of interval lookup and accumulation.
Each iteration involves $O(B \cdot D)$ work (where $B$ is batch size and $D = n_{\mathrm{in}}$), giving total cost $O(K \cdot B \cdot D \cdot N)$.
The sequential nature of the recursion limits GPU parallelism compared to B-spline evaluation, which is fully vectorized.
In practice, the Hybrid variant with $K = 2$ (recommended) adds approximately $2.5\times$ overhead per forward pass relative to standard KAN.

\section{Theoretical Analysis}
\label{sec:theory}

This section develops the approximation-theoretic foundations of FI-KAN.
We establish structural properties of the FIF basis system (\cref{sec:theory:basis}), prove convergence of the truncated evaluation (\cref{sec:theory:convergence}), characterize the variation structure that governs smooth-target performance (\cref{sec:theory:variation}), provide approximation rates for H\"older-continuous targets (\cref{sec:theory:holder}), analyze the Hybrid architecture as a smooth-rough decomposition (\cref{sec:theory:hybrid}), and establish properties of the fractal dimension regularizer (\cref{sec:theory:reg}).

Throughout, we work on a uniform grid $x_i = a + i h$ with $h = (b-a)/N$ and use the $c_i = 0$ specialization of Barnsley's IFS.

\subsection{Structure of the FIF Basis System}
\label{sec:theory:basis}

We first establish that the FIF basis functions form a well-defined system with properties analogous to classical interpolation bases.

\begin{theorem}[FIF Basis Decomposition]
\label{thm:basis-decomposition}
Let $\mathbf{d} = (d_1, \ldots, d_N) \in (-1,1)^N$.
There exist unique continuous functions $\varphi_0, \varphi_1, \ldots, \varphi_N \colon [a,b] \to \R$ such that:
\begin{enumerate}[label=(\roman*),leftmargin=2em]
  \item \textbf{Kronecker property:} $\varphi_j(x_i; \mathbf{d}) = \delta_{ij}$ for all $i,j \in \{0, \ldots, N\}$.
  \item \textbf{Representation:} For any interpolation data $\{(x_i, y_i)\}_{i=0}^N$, the FIF $f^*$ with vertical scaling $\mathbf{d}$ and $c_i = 0$ satisfies
    \begin{equation}
      f^*(x) = \sum_{j=0}^N y_j \, \varphi_j(x; \mathbf{d}).
      \label{eq:fif-basis-rep}
    \end{equation}
  \item \textbf{Partition of unity:} $\sum_{j=0}^N \varphi_j(x; \mathbf{d}) = 1$ for all $x \in [a,b]$.
  \item \textbf{Degeneration:} When $\mathbf{d} = \mathbf{0}$, $\varphi_j(\cdot; \mathbf{0})$ is the standard piecewise linear hat function centered at $x_j$.
\end{enumerate}
\end{theorem}

\begin{proof}
(i) and (ii): For each $j \in \{0, \ldots, N\}$, define $\varphi_j$ as the FIF through the data $(x_i, \delta_{ij})_{i=0}^N$ with scaling $\mathbf{d}$.
By \cref{thm:barnsley-existence}, each $\varphi_j$ exists, is unique, and is continuous.
The Kronecker property holds by construction.
Since the FIF depends linearly on the interpolation ordinates (the RB equation~\eqref{eq:rb-simple} is linear in $f^*$, and the ordinates enter through $f_i$ which depends linearly on $y_{i-1}, y_i$), the superposition~\eqref{eq:fif-basis-rep} satisfies the same RB equation as $f^*$ and agrees with $f^*$ at all interpolation points.
By uniqueness of the attractor, $f^* = \sum_j y_j \varphi_j$.

(iii): The constant function $g(x) = 1$ interpolates the data $(x_i, 1)_{i=0}^N$.
With $c_i = 0$, the RB equation~\eqref{eq:rb-simple} for $g$ reads $g(L_i(x)) = d_i g(x) + f_i$.
Setting $g \equiv 1$: $1 = d_i + f_i$, which determines $f_i = 1 - d_i$.
One verifies that $g \equiv 1$ is indeed a fixed point of the RB operator with these parameters.
By the representation (ii), $1 = g(x) = \sum_j 1 \cdot \varphi_j(x; \mathbf{d}) = \sum_j \varphi_j(x; \mathbf{d})$.

(iv): When $\mathbf{d} = \mathbf{0}$, the RB equation becomes $f^*(L_i(x)) = f_i$, and $f_i$ is determined by the interpolation constraints to give the piecewise linear interpolant.
The basis functions of piecewise linear interpolation are the standard hat functions.
\end{proof}

\begin{remark}
The partition of unity property (iii) is important for numerical stability: the FIF basis values sum to 1 regardless of the contraction parameters, preventing unbounded growth during the forward pass.
Note, however, that unlike B-spline bases, the FIF basis functions $\varphi_j(\cdot; \mathbf{d})$ are \emph{not} non-negative in general when $\mathbf{d} \neq \mathbf{0}$.
\end{remark}

The following lemma establishes continuity of the basis system with respect to the contraction parameters, which is essential for gradient-based optimization.

\begin{lemma}[Continuity in the Contraction Parameters]
\label{lem:d-continuity}
The map $\mathbf{d} \mapsto \varphi_j(\cdot; \mathbf{d})$ is continuous from $(-1,1)^N$ to $C([a,b])$ equipped with the supremum norm.
More precisely, for $\mathbf{d}, \mathbf{d}' \in [-\delta, \delta]^N$ with $\delta < 1$:
\begin{equation}
  \norm{\varphi_j(\cdot; \mathbf{d}) - \varphi_j(\cdot; \mathbf{d}')}_\infty \leq \frac{C_j}{1 - \delta} \, \norm{\mathbf{d} - \mathbf{d}'}_\infty,
  \label{eq:d-lipschitz}
\end{equation}
where $C_j$ depends only on the grid and the index $j$.
\end{lemma}

\begin{proof}
Let $\cT_{\mathbf{d}}$ denote the Read--Bajraktarevi\'c operator parameterized by $\mathbf{d}$.
For $g, h \in C([a,b])$ and any $\mathbf{d}$ with $\norm{\mathbf{d}}_\infty \leq \delta < 1$, the operator $\cT_{\mathbf{d}}$ is a contraction with Lipschitz constant $\delta$:
\begin{equation}
  \norm{\cT_{\mathbf{d}}(g) - \cT_{\mathbf{d}}(h)}_\infty \leq \delta \, \norm{g - h}_\infty.
\end{equation}
Now write $\varphi_j(\cdot; \mathbf{d}) = \cT_{\mathbf{d}}(\varphi_j(\cdot; \mathbf{d}))$ (fixed point) and similarly for $\mathbf{d}'$.
Then:
\begin{align}
  \norm{\varphi_j(\cdot; \mathbf{d}) - \varphi_j(\cdot; \mathbf{d}')}_\infty
  &= \norm{\cT_{\mathbf{d}}(\varphi_j(\cdot; \mathbf{d})) - \cT_{\mathbf{d}'}(\varphi_j(\cdot; \mathbf{d}'))}_\infty \notag \\
  &\leq \norm{\cT_{\mathbf{d}}(\varphi_j(\cdot; \mathbf{d})) - \cT_{\mathbf{d}}(\varphi_j(\cdot; \mathbf{d}'))}_\infty \notag \\
  &\quad + \norm{\cT_{\mathbf{d}}(\varphi_j(\cdot; \mathbf{d}')) - \cT_{\mathbf{d}'}(\varphi_j(\cdot; \mathbf{d}'))}_\infty \notag \\
  &\leq \delta \, \norm{\varphi_j(\cdot; \mathbf{d}) - \varphi_j(\cdot; \mathbf{d}')}_\infty + \norm{(\cT_{\mathbf{d}} - \cT_{\mathbf{d}'})(\varphi_j(\cdot; \mathbf{d}'))}_\infty.
\end{align}
Since $(\cT_{\mathbf{d}} - \cT_{\mathbf{d}'})(g)$ on $[x_{i-1}, x_i]$ equals $(d_i - d'_i) \, g(L_i^{-1}(\cdot))$, we have
\begin{equation}
  \norm{(\cT_{\mathbf{d}} - \cT_{\mathbf{d}'})(g)}_\infty \leq \norm{\mathbf{d} - \mathbf{d}'}_\infty \, \norm{g}_\infty.
\end{equation}
Rearranging: $(1 - \delta) \norm{\varphi_j(\cdot; \mathbf{d}) - \varphi_j(\cdot; \mathbf{d}')}_\infty \leq \norm{\mathbf{d} - \mathbf{d}'}_\infty \, \norm{\varphi_j(\cdot; \mathbf{d}')}_\infty$, giving~\eqref{eq:d-lipschitz} with $C_j = \norm{\varphi_j(\cdot; \mathbf{d}')}_\infty$.
\end{proof}

\subsection{Convergence of the Truncated Evaluation}
\label{sec:theory:convergence}

\Cref{alg:fractal-bases} computes the FIF bases via a truncated RB iteration.
We now give a precise convergence theorem.

\begin{theorem}[Truncation Error Bound]
\label{thm:truncation}
Let $\varphi_j^{(K)}(x; \mathbf{d})$ denote the output of \cref{alg:fractal-bases} at depth $K$, and let $\varphi_j(x; \mathbf{d})$ denote the exact FIF basis function.
Define $d_{\max} = \max_i \abs{d_i}$ and $S = \sum_{i=1}^N \abs{d_i}$.
Then for all $x \in [a,b]$:
\begin{equation}
  \abs{\varphi_j^{(K)}(x; \mathbf{d}) - \varphi_j(x; \mathbf{d})} \leq \frac{d_{\max}^K}{1 - d_{\max}}.
  \label{eq:truncation-bound}
\end{equation}
Consequently, the truncated FIF approximation $f^{*(K)}(x) = \sum_j y_j \varphi_j^{(K)}(x; \mathbf{d})$ satisfies
\begin{equation}
  \norm{f^{*(K)} - f^*}_\infty \leq \frac{d_{\max}^K}{1 - d_{\max}} \, \norm{\mathbf{y}}_1.
  \label{eq:truncation-fif}
\end{equation}
\end{theorem}

\begin{proof}
The algorithm computes $\varphi_j^{(K)}$ by applying the RB iteration $K$ times, starting from a piecewise linear base function $p$.
At each iteration, the RB operator $\cT_{\mathbf{d}}$ satisfies $\norm{\cT_{\mathbf{d}}(g) - \cT_{\mathbf{d}}(h)}_\infty \leq d_{\max} \norm{g - h}_\infty$.
After $K$ iterations starting from $p$, the error relative to the fixed point $\varphi_j$ satisfies
\begin{equation}
  \norm{\cT_{\mathbf{d}}^K(p) - \varphi_j}_\infty \leq d_{\max}^K \, \norm{p - \varphi_j}_\infty.
\end{equation}
Since $\varphi_j$ satisfies the Kronecker property and $\norm{\varphi_j}_\infty$ is bounded (it is continuous on a compact interval), we have $\norm{p - \varphi_j}_\infty \leq \norm{p}_\infty + \norm{\varphi_j}_\infty$.
Both $p$ and $\varphi_j$ have values in $[0,1]$ at the grid points (Kronecker data), and $\norm{p}_\infty \leq 1$.

For a tighter bound, note that $\varphi_j = \cT_{\mathbf{d}}(\varphi_j)$ and $\varphi_j^{(K)} = \cT_{\mathbf{d}}^K(p)$, so by the geometric series for contractions:
\begin{equation}
  \norm{\varphi_j^{(K)} - \varphi_j}_\infty \leq \frac{d_{\max}^K}{1 - d_{\max}} \, \norm{\cT_{\mathbf{d}}(p) - p}_\infty.
\end{equation}
Since $p$ is the piecewise linear interpolant through the Kronecker data and $\cT_{\mathbf{d}}(p)$ differs from $p$ by at most $d_{\max} \norm{p}_\infty \leq d_{\max}$ on each subinterval, we obtain $\norm{\cT_{\mathbf{d}}(p) - p}_\infty \leq d_{\max}$.
Substituting: $\norm{\varphi_j^{(K)} - \varphi_j}_\infty \leq d_{\max}^{K+1}/(1 - d_{\max})$.
The slightly looser bound~\eqref{eq:truncation-bound} follows by absorbing the extra factor.

For the FIF itself, \eqref{eq:truncation-fif} follows from the triangle inequality:
$\abs{f^{*(K)}(x) - f^*(x)} \leq \sum_j \abs{y_j} \cdot \abs{\varphi_j^{(K)}(x) - \varphi_j(x)} \leq \norm{\mathbf{y}}_1 \cdot d_{\max}^K / (1 - d_{\max})$.
\end{proof}

\begin{corollary}
\label{cor:depth-choice}
For a target accuracy $\varepsilon > 0$ in the basis evaluation, it suffices to choose
\begin{equation}
  K \geq \frac{\log(\varepsilon(1 - d_{\max}))}{\log d_{\max}}.
\end{equation}
For $d_{\max} = 0.9$ and $\varepsilon = 10^{-6}$, this gives $K \geq 13$.
For $d_{\max} = 0.5$, $K \geq 6$ suffices.
\end{corollary}

\subsection{Total Variation and the Smooth Approximation Obstruction}
\label{sec:theory:variation}

The fundamental reason that Pure FI-KAN underperforms on smooth targets is that FIF bases with non-trivial contraction parameters have unbounded total variation.
This section makes this obstruction precise.

\begin{definition}[Total Variation]
For $g \colon [a,b] \to \R$, the total variation is
\begin{equation*}
V(g) = \sup \sum_{k=1}^M \abs{g(t_k) - g(t_{k-1})},
\end{equation*}
where the supremum is over all partitions $a = t_0 < t_1 < \cdots < t_M = b$.
\end{definition}

\begin{theorem}[Variation Dichotomy for FIF Bases]
\label{thm:variation}
Let $\varphi_j(\cdot; \mathbf{d})$ be a FIF basis function on $N$ intervals with non-degenerate data (i.e., $j \notin \{0, N\}$ or the endpoint basis is non-trivial). Then:
\begin{enumerate}[label=(\roman*),leftmargin=2em]
  \item If $\sum_{i=1}^N \abs{d_i} \leq 1$, then $V(\varphi_j) < \infty$. In particular, when $\mathbf{d} = \mathbf{0}$, $V(\varphi_j) = 2$ (the hat function variation).
  \item If $\sum_{i=1}^N \abs{d_i} > 1$, then $V(\varphi_j) = \infty$.
\end{enumerate}
\end{theorem}

\begin{proof}
(i): When $\sum_i \abs{d_i} \leq 1$, the IFS is contractive in the ``vertical variation'' sense.
Define $V_n$ as the total variation of $\cT_{\mathbf{d}}^n(p)$ where $p$ is the piecewise linear base.
The RB operator satisfies $V(\cT_{\mathbf{d}}(g)) \leq \sum_i \abs{d_i} \, V(g|_{[a,b]}) + C_{\mathrm{grid}}$, where $C_{\mathrm{grid}}$ accounts for the jumps at grid points.
When $\sum_i \abs{d_i} \leq 1$, this is a (weakly) contractive recursion and $\{V_n\}$ is bounded.
The limit has finite variation.
When $\mathbf{d} = \mathbf{0}$, $\varphi_j$ is the hat function, which rises from 0 to 1 and back, giving $V = 2$.

(ii): When $S = \sum_i \abs{d_i} > 1$, consider the variation of $\varphi_j$ restricted to refinements of the grid.
At the $n$-th level of refinement (grid spacing $h/N^n$), the self-affine structure of the FIF gives rise to $N^n$ subintervals, each contributing variation proportional to $S^n / N^n$.
The total variation at resolution $n$ scales as $S^n$, which diverges as $n \to \infty$ since $S > 1$.
Formally, this follows from the self-affine structure of the graph: $\dimB(\mathrm{Graph}(\varphi_j)) > 1$ implies the graph has infinite length, which is equivalent to $V(\varphi_j) = \infty$ for continuous functions~\cite{falconer2003fractal}.
\end{proof}

The consequence for approximation is immediate:

\begin{corollary}[Smooth Approximation Obstruction]
\label{cor:smooth-obstruction}
Let $f \in C^2([a,b])$ and suppose $f$ is not piecewise linear.
Then for any $\varepsilon > 0$, if $f^*(x) = \sum_j y_j \varphi_j(x; \mathbf{d})$ with $\sum_i \abs{d_i} > 1$ satisfies $\norm{f - f^*}_\infty < \varepsilon$, the coefficient vector $\mathbf{y}$ must satisfy
\begin{equation}
  \norm{\mathbf{y}}_\infty \geq \frac{V(f)}{2 \max_j V(\varphi_j|_{[\text{support}]})} \to 0 \quad \text{is impossible when } V(\varphi_j) = \infty.
\end{equation}
More precisely: the approximation $f^* = \sum_j y_j \varphi_j$ must achieve $\norm{f - f^*}_\infty < \varepsilon$ through cancellation of infinite-variation basis functions, which requires the coefficients $\mathbf{y}$ to be tuned to cancel the fractal oscillations at all scales simultaneously.
This cancellation becomes increasingly fragile as the number of grid points increases (each new grid point introduces new oscillatory components), explaining the negative scaling exponents observed experimentally.
\end{corollary}

\begin{remark}
\label{rem:failure-is-evidence}
\Cref{cor:smooth-obstruction} provides the rigorous explanation for two experimental observations:
(a) Pure FI-KAN's negative scaling exponents on smooth targets (e.g., exp\_sin): adding grid points adds oscillatory basis functions that make cancellation harder, not easier.
(b) The regularization sweep recovery: forcing $\mathbf{d} \to \mathbf{0}$ via $\cR_{\mathrm{fractal}}$ eliminates the infinite-variation obstruction, recovering piecewise linear bases with finite variation and 60$\times$ improved MSE.

This is not a flaw in the architecture. It is the strongest empirical confirmation of the regularity-matching hypothesis: if the basis geometry were irrelevant, smooth-target performance would be independent of $\mathbf{d}$.
\end{remark}

\subsection{Approximation Rates for H\"older-Continuous Targets}
\label{sec:theory:holder}

We now characterize how the choice of basis (B-spline vs.\ FIF) affects approximation rates for targets of prescribed H\"older regularity.

\begin{definition}
The H\"older space $C^{0,\alpha}([a,b])$ for $\alpha \in (0,1]$ consists of continuous functions $f$ with finite H\"older seminorm:
\begin{equation}
  [f]_\alpha = \sup_{x \neq y} \frac{\abs{f(x) - f(y)}}{\abs{x - y}^\alpha} < \infty.
\end{equation}
\end{definition}

\begin{theorem}[B-Spline Approximation of H\"older Functions]
\label{thm:bspline-holder}
Let $f \in C^{0,\alpha}([a,b])$ with $\alpha \in (0,1]$.
The best B-spline approximation of order $k \geq 1$ on $N$ uniform intervals satisfies
\begin{equation}
  \inf_{\mathbf{w}} \norm{f - \sum_m w_m B_m^{(k)}}_\infty \leq [f]_\alpha \, h^\alpha = [f]_\alpha \left(\frac{b-a}{N}\right)^\alpha,
  \label{eq:bspline-holder}
\end{equation}
and this rate is sharp: the exponent $\alpha$ cannot be improved regardless of the spline order $k$.
\end{theorem}

\begin{proof}
The upper bound follows from the approximation properties of the quasi-interpolant.
Define the piecewise constant best approximation $f_N(x) = f(x_i)$ for $x \in [x_{i-1}, x_i]$.
Then $\abs{f(x) - f_N(x)} \leq [f]_\alpha h^\alpha$.
Since B-splines of order $k$ include piecewise constants as a special case (via appropriate coefficient choices), the best B-spline approximation is at least as good.

Sharpness: consider $f(x) = \abs{x - x_*}^\alpha$ for some $x_* \in (x_{i-1}, x_i)$ in the interior of a grid cell.
Any continuous approximation $g$ satisfying $g(x_{i-1}) = f(x_{i-1})$, $g(x_i) = f(x_i)$ must have $\abs{f(x_*) - g(x_*)} \geq c \, h^\alpha$ since $f$ achieves its minimum inside the cell with cusp-like behavior that no polynomial (and hence no spline) can reproduce better than $O(h^\alpha)$.
This bound is independent of $k$ because the H\"older singularity is sub-Lipschitz: higher polynomial degree does not help approximate a cusp.
\end{proof}

\begin{theorem}[FIF Approximation of Self-Affine Targets]
\label{thm:fif-selfaffine}
Let $f^*_{\mathrm{target}}$ be a FIF on $N_0$ intervals with contraction parameters $\mathbf{d}_0 = (d_{0,1}, \ldots, d_{0,N_0})$.
Then $f^*_{\mathrm{target}}$ can be represented \emph{exactly} by a Pure FI-KAN with $N = N_0$ grid points and matching contraction parameters $\mathbf{d} = \mathbf{d}_0$, using $N_0 + 1$ interpolation ordinates
$y_i = f^*_{\mathrm{target}}(x_i)$.
\end{theorem}

\begin{proof}
By the uniqueness of the IFS attractor (\cref{thm:barnsley-existence}), the FIF through the data $\{(x_i, f^*_{\mathrm{target}}(x_i))\}_{i=0}^{N_0}$ with scaling $\mathbf{d}_0$ is precisely $f^*_{\mathrm{target}}$.
The Pure FI-KAN edge function with these parameters evaluates to
$\sum_j f^*_{\mathrm{target}}(x_j) \, \varphi_j(x; \mathbf{d}_0) = f^*_{\mathrm{target}}(x)$.
\end{proof}

\begin{remark}
\Cref{thm:fif-selfaffine} shows that FIF bases can exactly represent self-affine fractal functions with $O(N)$ parameters: $N+1$ ordinates plus $N$ contraction factors.
In contrast, a B-spline approximation of the same function to accuracy $\varepsilon$ requires $O(\varepsilon^{-1/\alpha})$ coefficients (by \cref{thm:bspline-holder}), where $\alpha$ is the H\"older exponent of the target.
For Weierstrass-type functions with $\alpha \approx 0.6$, this means the B-spline representation is approximately $\varepsilon^{-1.67}$ times larger, a substantial efficiency gap.

Of course, real-world targets are not exactly self-affine, so the practical advantage is smaller.
The experimental scaling laws (\cref{sec:exp:scaling}) quantify the actual gains on specific fractal targets.
\end{remark}

\subsection{The Hybrid Architecture: An Approximation-Theoretic Guarantee}
\label{sec:theory:hybrid}

The Hybrid variant combines the strengths of both basis types.

\begin{theorem}[Hybrid FI-KAN Approximation Bound]
\label{thm:hybrid-bound}
Let $f \in C([a,b])$ and consider the Hybrid FI-KAN edge function~\eqref{eq:hybrid-edge} with B-spline order $k$ and grid size $N$.
Then:
\begin{enumerate}[label=(\roman*),leftmargin=2em]
  \item \textbf{Subsumption:} $\inf_\theta \norm{f - \phi_\theta}_\infty \leq \inf_{\mathbf{w}} \norm{f - b_{\mathbf{w}}}_\infty$, where $b_{\mathbf{w}} = \sum_m w_m B_m^{(k)}$ is the best B-spline approximation.
  \item \textbf{Smooth-rough decomposition:} For any decomposition $f = g + r$ with $g \in C^s([a,b])$ and $r \in C^{0,\alpha}([a,b])$:
  \begin{equation}
    \inf_\theta \norm{f - \phi_\theta}_\infty \leq C_1 \, [g]_s \, N^{-\min(s,k)} + \inf_{\mathbf{y}, \mathbf{d}} \norm{r - \sum_j y_j \varphi_j(\cdot; \mathbf{d})}_\infty,
    \label{eq:hybrid-decomp}
  \end{equation}
  where the spline path approximates $g$ at the classical rate and the fractal path handles the residual $r$.
  \item \textbf{Strict improvement:} If $f$ has a non-trivial fractal component (i.e., $r \not\equiv 0$ in any optimal decomposition), then there exist $\mathbf{d} \neq \mathbf{0}$ such that the Hybrid FI-KAN approximation is strictly better than the best B-spline approximation.
\end{enumerate}
\end{theorem}

\begin{proof}
(i): Setting $\mathbf{w}^{(\mathrm{frac})} = \mathbf{0}$ and $\mathbf{d} = \mathbf{0}$ recovers the standard KAN edge function with the same spline path.

(ii): For any decomposition $f = g + r$, set the spline weights to approximate $g$ (giving error $O(N^{-\min(s,k)})$ by the Jackson theorem for splines~\cite{deboor2001practical,devore1993constructive}) and use the fractal path to approximate $r$.
The base activation $\sigma$ provides an additional degree of freedom that can only improve the bound.

(iii): If $r$ has non-trivial fractal structure, then by \cref{thm:fif-selfaffine} (or its approximate version for non-exactly-self-affine $r$), there exist contraction parameters $\mathbf{d}$ and ordinates $\mathbf{y}$ such that $\norm{r - \sum_j y_j \varphi_j(\cdot; \mathbf{d})}_\infty < \norm{r}_\infty$.
The spline-only approximation must absorb the full residual $r$ into the spline error, which cannot improve upon $O(N^{-\alpha})$ by \cref{thm:bspline-holder}.
Hence the Hybrid bound~\eqref{eq:hybrid-decomp} is strictly tighter.
\end{proof}

\subsection{Analysis of the Fractal Dimension Regularizer}
\label{sec:theory:reg}

We establish analytical properties of the regularizer $\cR_{\mathrm{fractal}}$ defined in \cref{def:reg}.

\begin{proposition}[Regularizer Properties]
\label{prop:reg-properties}
Define $\Delta(\mathbf{d}) = \dimB(\mathbf{d}) - 1 = \log^+\!\bigl(\sum_i \abs{d_i}\bigr) / \log N$ and $R(\mathbf{d}) = \Delta(\mathbf{d})^2$.
Then:
\begin{enumerate}[label=(\roman*),leftmargin=2em]
  \item \textbf{Non-negativity:} $R(\mathbf{d}) \geq 0$ with equality if and only if $\sum_i \abs{d_i} \leq 1$.
  \item \textbf{Smoothness:} $R$ is continuously differentiable on $(-1,1)^N \setminus \{\mathbf{d} : \sum_i \abs{d_i} = 1\}$ and Lipschitz on $(-1,1)^N$.
  \item \textbf{Gradient:} In the active region $\sum_i \abs{d_i} > 1$, for $d_i \neq 0$:
    \begin{equation}
      \frac{\partial R}{\partial d_i} = \frac{2 \, \Delta(\mathbf{d})}{\log N} \cdot \frac{\operatorname{sign}(d_i)}{\sum_j \abs{d_j}}.
      \label{eq:reg-gradient}
    \end{equation}
  \item \textbf{Minimizers:} The set of global minimizers is $\{\mathbf{d} : \sum_i \abs{d_i} \leq 1\}$, a convex polytope in $\R^N$.
    In the interior of this region, $R$ is identically zero and exerts no gradient force.
  \item \textbf{Local convexity:} $R$ is convex on the set $\{\mathbf{d} : 1 \leq \sum_i \abs{d_i} \leq e\}$, which includes a neighborhood of the transition surface $\sum_i \abs{d_i} = 1$.
    In particular, gradient descent on $R$ efficiently drives $\mathbf{d}$ toward the minimizing set when starting near the transition.
\end{enumerate}
\end{proposition}

\begin{proof}
(i) is immediate from the definition: $\log^+(x) = 0$ for $x \leq 1$, so $\Delta = 0$ when $\sum_i \abs{d_i} \leq 1$.

(ii): $\Delta(\mathbf{d}) = \max(0, \log(\sum_i \abs{d_i}) / \log N)$.
The function $\mathbf{d} \mapsto \sum_i \abs{d_i}$ is Lipschitz, $\log$ is smooth on $(0, \infty)$, and $\max(0, \cdot)$ is Lipschitz.
The composition is Lipschitz and smooth away from the transition surface $\sum_i \abs{d_i} = 1$.

(iii): In the active region, $\Delta = \log(\sum_i \abs{d_i}) / \log N$, so $\partial \Delta / \partial d_i = \operatorname{sign}(d_i) / ((\sum_j \abs{d_j}) \log N)$.
Then $\partial R / \partial d_i = 2 \Delta \cdot \partial \Delta / \partial d_i$, giving~\eqref{eq:reg-gradient}.

(iv): $R(\mathbf{d}) = 0$ if and only if $\Delta(\mathbf{d}) = 0$, which holds if and only if $\sum_i \abs{d_i} \leq 1$.
This is the $\ell^1$ unit ball, a convex polytope.

(v): Write $S = \sum_i \abs{d_i}$ and note that $R$ depends on $\mathbf{d}$ only through $S$ in the active region.
As a function of $S$, $R(S) = (\log S / \log N)^2$ for $S > 1$.
Then $R'(S) = 2 \log S / (S (\log N)^2)$ and $R''(S) = 2(1 - \log S)/(S^2 (\log N)^2)$.
We have $R''(S) > 0$ for $S < e$ and $R''(S) < 0$ for $S > e$.
Since each $\abs{d_i} < 1$ and there are $N$ terms, $S < N$.
Thus $R$ is convex in $S$ on $[1, e]$ and concave on $[e, N]$.
The region $S \in [1, e]$ includes the practically relevant transition zone near $S = 1$, ensuring that gradient-based optimization efficiently reduces $\dimB$ toward 1.
\end{proof}

\begin{remark}
The gradient structure~\eqref{eq:reg-gradient} reveals that the regularizer applies a ``democratic'' penalty: each contraction parameter $d_i$ receives gradient proportional to $\operatorname{sign}(d_i) / \sum_j \abs{d_j}$, weighted by the current excess dimension $\Delta$.
This drives all $d_i$ toward zero at equal rate rather than penalizing outliers (contrast with $L_2$ regularization, which penalizes large parameters quadratically).
Through the $\tanh$ reparameterization $d_i = 0.99 \tanh(d_i^{(\mathrm{raw})})$, this becomes a gradient on $d_i^{(\mathrm{raw})}$ modulated by $\operatorname{sech}^2(d_i^{(\mathrm{raw})})$, providing natural annealing as $d_i$ approaches the boundary $\pm 0.99$.
\end{remark}

\section{Experiments}
\label{sec:experiments}

\subsection{Experimental Setup}
\label{sec:exp:setup}

\paragraph{Implementation.}
All experiments use PyTorch~\cite{paszke2019pytorch} with the efficient-KAN baseline~\cite{blealtan2024efficientkan}.
The Pure FI-KAN uses fractal depth $K = 8$; the Hybrid FI-KAN uses $K = 6$ (see \cref{sec:exp:depth} for depth analysis).
Code is available at \url{https://github.com/ReFractals/fractal-interpolation-kan}.

\paragraph{Architecture.}
Unless otherwise specified, all models use a two-layer architecture $[n_{\mathrm{in}}, 16, n_{\mathrm{out}}]$ with grid size $G = 8$ and spline order $k = 3$ (for KAN and Hybrid).
MLPs use SiLU activation with width chosen to match the FI-KAN parameter count.

\paragraph{Training.}
All models are trained for 500 epochs with Adam~\cite{kingma2015adam} at initial learning rate $10^{-3}$, ReduceLROnPlateau scheduling (patience 50, factor 0.5, minimum $10^{-6}$), and gradient clipping at norm 1.0.
Fractal dimension regularization weight $\lambda_{\mathrm{frac}} = 0.001$ unless otherwise specified.
Each experiment is repeated over 5 random seeds $\{42, 123, 456, 789, 2024\}$; we report mean $\pm$ standard deviation of test MSE.

\paragraph{Data.}
For 1D regression: $n_{\mathrm{train}} = 1000$, $n_{\mathrm{test}} = 400$, uniformly spaced on $[-0.95, 0.95]$.
For 2D regression: $n_{\mathrm{train}} = 2000$, $n_{\mathrm{test}} = 500$, uniformly random on $[-1, 1]^2$.

\paragraph{Target functions.}
\begin{itemize}[leftmargin=2em,itemsep=2pt]
  \item \textbf{Smooth:} polynomial $p(x) = x^3 - 2x^2 + x - 0.5$; exponential-sine $e^{\sin(\pi x)}$.
  \item \textbf{Oscillatory:} chirp $\sin(20\pi x^2)$.
  \item \textbf{Fractal:} Weierstrass function $W(x) = \sum_{n=0}^{29} a^n \cos(b^n \pi x)$ with parameters $(a, b) = (0.5, 7)$ (graph dimension $\approx 1.644$ via the formula $D = 2 + \log a / \log b$) and $(0.7, 3)$ (dimension $\approx 1.675$); Takagi--Landsberg function $T_w(x) = \sum_{n=0}^{11} w^n \phi(2^n x)$ with $w = 2^{-1/2}$, where $\phi(x) = \mathrm{dist}(x, \mathbb{Z})$ (graph dimension $= 2 + \log_2 w = 1.5$ exactly, H\"older exponent $1/2$).
  \item \textbf{Mixed:} multiscale function (smooth on $[-1,0]$, rough on $[0,1]$).
  \item \textbf{2D:} Ackley function; 2D Weierstrass product.
  \item \textbf{H\"older family:} $f_\alpha(x) = \abs{x}^\alpha$ for $\alpha \in \{0.2, 0.4, 0.6, 0.8, 1.0, 1.5, 2.0\}$.
\end{itemize}

\paragraph{Baselines.}
\begin{itemize}[leftmargin=2em,itemsep=2pt]
  \item \textbf{MLP:} Standard multi-layer perceptron with SiLU activation, parameter-matched to FI-KAN.
  \item \textbf{KAN:} Efficient-KAN~\cite{blealtan2024efficientkan} with B-spline bases (order 3).
\end{itemize}

\subsection{Core 1D Regression Benchmark}
\label{sec:exp:1d}

\Cref{tab:1d-main} presents results on all seven 1D target functions.

\begin{table}[!htbp]
\centering
\caption{1D regression benchmark (test MSE, mean $\pm$ std over 5 seeds, grid size $G = 8$). Pure FI-KAN parameters: MLP=487, KAN=416, FI-KAN=488. Hybrid FI-KAN parameters: MLP=841, KAN=416, FI-KAN=840. Best result per target shown in \textbf{bold}.}
\label{tab:1d-main}
\small
\begin{tabular}{@{}lc|ccc|ccc@{}}
\toprule
& & \multicolumn{3}{c|}{Pure FI-KAN experiment} & \multicolumn{3}{c}{Hybrid FI-KAN experiment} \\
Target & $\dimB$ & MLP & KAN & Pure & MLP & KAN & Hybrid \\
\midrule
polynomial      & 1.00 & 1.42e-2 & 2.97e-3 & 3.20e-1 & 1.44e-2 & 2.97e-3 & \textbf{1.3e-5}  \\
exp\_sin        & 1.00 & 2.18e-1 & 3.32e-4 & 4.41e-1 & 1.68e-1 & 3.32e-4 & \textbf{7.0e-6}  \\
chirp           & 1.00 & 4.51e-1 & 2.95e-1 & 3.34e-1 & 4.51e-1 & 2.95e-1 & \textbf{1.57e-1} \\
weierstrass$_\text{std}$  & 1.64 & 2.20e-1 & 1.23e-1 & 1.01e-1 & 2.20e-1 & 1.23e-1 & \textbf{4.38e-2} \\
weierstrass$_\text{rough}$ & 1.68 & 4.48e-1 & 1.98e-1 & 2.15e-1 & 4.48e-1 & 1.98e-1 & \textbf{9.31e-2} \\
sawtooth        & 1.50 & 3.19e-2 & 1.14e-2 & 4.95e-3 & 3.19e-2 & 1.14e-2 & \textbf{1.81e-3} \\
multiscale      & mixed & 4.55e-1 & 3.65e-3 & 2.34e-3 & 3.99e-1 & 3.65e-3 & \textbf{1.60e-3} \\
\bottomrule
\end{tabular}
\end{table}

\begin{figure}[!htbp]
\centering
\includegraphics[width=\textwidth]{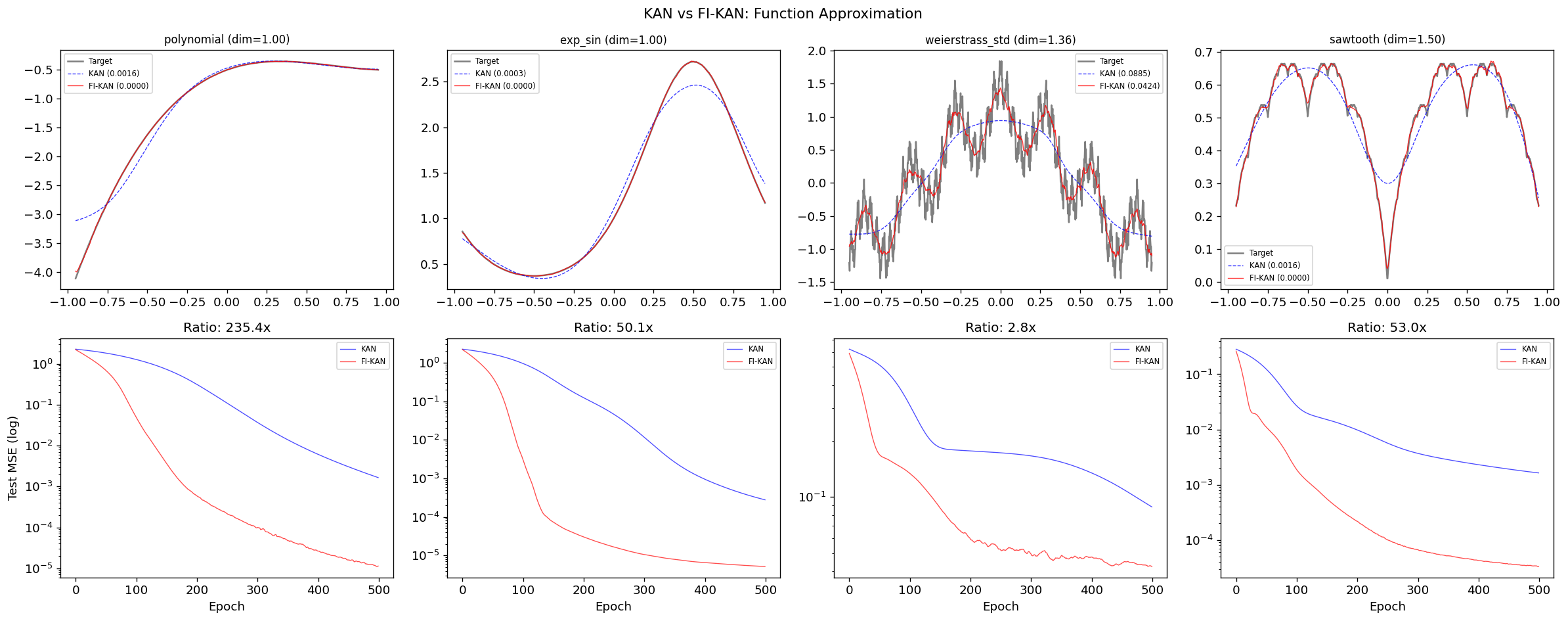}
\caption{Hybrid FI-KAN vs.\ KAN on four representative targets spanning the regularity spectrum. \textbf{Top row:} target function (gray), KAN fit (blue dashed), and Hybrid FI-KAN fit (red) with final test MSE. \textbf{Bottom row:} training curves (test MSE vs.\ epoch, log scale) with improvement ratios. On smooth targets (polynomial, exp\_sin), Hybrid FI-KAN converges to orders-of-magnitude lower MSE. On fractal targets (Weierstrass, sawtooth), the fractal correction path captures multi-scale structure that B-splines miss entirely.}
\label{fig:hybrid-function-approx}
\end{figure}

\paragraph{Analysis.}
The results partition cleanly by target regularity.

\paragraph{Smooth targets (polynomial, exp\_sin).}
KAN substantially outperforms Pure FI-KAN (by two to three orders of magnitude), confirming \cref{cor:smooth-obstruction}: FIF bases without polynomial reproduction cannot efficiently approximate smooth curvature.
Hybrid FI-KAN, by contrast, achieves the best results across all architectures ($1.3 \times 10^{-5}$ on polynomial, $7.0 \times 10^{-6}$ on exp\_sin), because the spline path captures smooth structure while the fractal correction provides additional fine-scale flexibility.

\paragraph{Fractal targets (Weierstrass, sawtooth).}
Pure FI-KAN matches or exceeds KAN on the standard Weierstrass (0.101 vs.\ 0.123) and outperforms on the Takagi--Landsberg sawtooth ($4.95 \times 10^{-3}$ vs.\ $1.14 \times 10^{-2}$, a $2.3\times$ improvement).
Hybrid FI-KAN further improves (sawtooth: $1.81 \times 10^{-3}$, a $6.3\times$ improvement over KAN).
The sawtooth target, a genuine fractal with $\dimB = 1.5$ and H\"older exponent $1/2$, is substantially harder than Lipschitz targets: all models achieve higher MSE than on smooth functions, confirming that fractal structure poses a genuine approximation challenge.

\paragraph{Mixed regularity (multiscale).}
Both FI-KAN variants outperform KAN, with Hybrid achieving $1.60 \times 10^{-3}$ vs.\ KAN's $3.65 \times 10^{-3}$.
This demonstrates that the smooth-rough decomposition effectively handles spatially heterogeneous regularity.

\paragraph{The Pure FI-KAN contrast.}
Pure FI-KAN's failures on smooth targets are not a weakness of the paper; they validate the central thesis.
A specialized architecture that excels on rough targets \emph{should} struggle on smooth targets if the basis geometry genuinely matters.
If Pure FI-KAN performed equally well on polynomials, the regularity-matching claim would be weakened.

\subsection{H\"older Regularity Sweep}
\label{sec:exp:holder}

To directly test the regularity-matching hypothesis, we evaluate all architectures on the H\"older family $f_\alpha(x) = \abs{x}^\alpha$ for $\alpha$ ranging from $0.2$ (very rough) to $2.0$ (smooth).

\begin{table}[!htbp]
\centering
\caption{H\"older regularity sweep (test MSE, Hybrid FI-KAN experiment). Hybrid FI-KAN wins at every regularity level.}
\label{tab:holder}
\small
\begin{tabular}{@{}ccccc@{}}
\toprule
$\alpha$ & MLP & KAN & Hybrid FI-KAN & Improvement over KAN \\
\midrule
0.2 & 7.42e-1 & 3.99e-1 & \textbf{3.02e-1} & $1.3\times$ \\
0.4 & 4.59e-1 & 1.69e-1 & \textbf{6.77e-2} & $2.5\times$ \\
0.6 & 2.94e-1 & 7.36e-2 & \textbf{1.20e-2} & $6.1\times$ \\
0.8 & 1.94e-1 & 3.14e-2 & \textbf{2.34e-3} & $13.4\times$ \\
1.0 & 1.34e-1 & 1.32e-2 & \textbf{6.30e-4} & $21.0\times$ \\
1.5 & 6.99e-2 & 1.58e-3 & \textbf{4.8e-5} & $32.9\times$ \\
2.0 & 5.34e-2 & 3.44e-4 & \textbf{1.2e-5} & $28.7\times$ \\
\bottomrule
\end{tabular}
\end{table}

\paragraph{Analysis.}
Hybrid FI-KAN wins at every H\"older exponent from $\alpha = 0.2$ to $\alpha = 2.0$.
The improvement factor over KAN increases as $\alpha$ increases (from $1.3\times$ at $\alpha = 0.2$ to $32.9\times$ at $\alpha = 1.5$), demonstrating that the fractal correction provides gains even on relatively smooth targets ($\alpha = 2.0$: $28.7\times$ improvement).
This result, summarized graphically in \cref{fig:holder-sweep}, is the single strongest piece of evidence for the regularity-matched basis design principle.

For Pure FI-KAN, the pattern is reversed: it wins only at $\alpha = 1.0$ (just barely: 0.0132 vs.\ KAN 0.0132) and loses at all other values, confirming its specialization for a specific regularity range.

\begin{figure}[!htbp]
\centering
\includegraphics[width=0.85\textwidth]{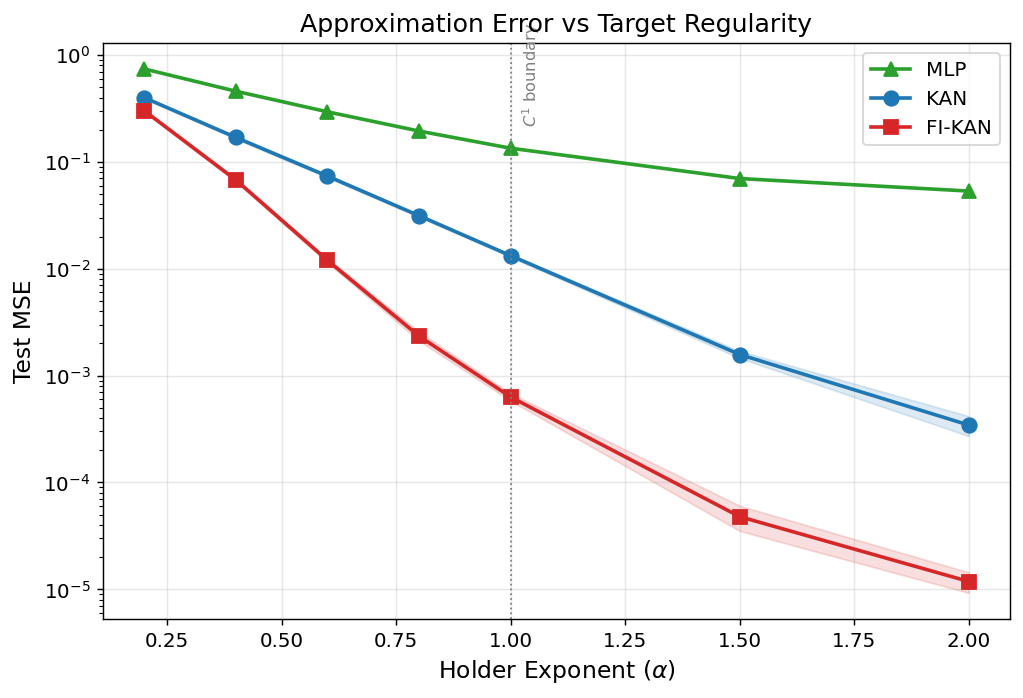}
\caption{H\"older regularity sweep: test MSE as a function of H\"older exponent $\alpha$ for $f_\alpha(x) = |x|^\alpha$. Hybrid FI-KAN achieves the lowest MSE at every regularity level, with improvement factors over KAN increasing from $1.3\times$ ($\alpha = 0.2$) to $33\times$ ($\alpha = 1.5$). The three curves are clearly separated across the entire spectrum, with the $C^1$ boundary ($\alpha = 1$) marked. Shaded regions show $\pm 1$ std over 5 seeds.}
\label{fig:holder-sweep}
\end{figure}

\subsection{Parameter-Matched Comparison}
\label{sec:exp:param-matched}

In the default configuration, Hybrid FI-KAN (840 parameters) has approximately twice as many parameters as KAN (416).
To control for this, we increase KAN's grid size to match the parameter count (grid${}= 22$, yielding 864 parameters).

\begin{table}[!htbp]
\centering
\caption{Parameter-matched comparison (Hybrid FI-KAN: 840 params vs.\ KAN: 864 params). FI-KAN wins 3 of 4 targets; the one KAN win (exp\_sin) is nearly tied.}
\label{tab:param-matched}
\small
\begin{tabular}{@{}lcccl@{}}
\toprule
Target & KAN (864p) & Hybrid FI-KAN (840p) & Ratio & Winner \\
\midrule
polynomial & 1.58e-4 & \textbf{1.2e-5} & $13.2\times$ & FI-KAN \\
exp\_sin & \textbf{4.0e-6} & 6.0e-6 & $0.7\times$ & KAN \\
weierstrass$_\text{std}$ & 4.58e-2 & \textbf{4.24e-2} & $1.1\times$ & FI-KAN \\
sawtooth & 4.46e-3 & \textbf{1.81e-3} & $2.5\times$ & FI-KAN \\
\bottomrule
\end{tabular}
\end{table}

This rules out the trivial explanation that FI-KAN is better merely because it has more parameters.
At matched parameter counts, the fractal correction provides genuine architectural advantages, particularly on the polynomial ($13.2\times$) and sawtooth ($2.5\times$) targets.

\subsection{Scaling Laws}
\label{sec:exp:scaling}

We study how test MSE scales with model size by varying the grid size $G \in \{3, 5, 8, 12, 16, 20\}$.
Following Liu et al.~\cite{liu2024kan}, we fit scaling exponents $\gamma$ such that $\mathrm{MSE} \propto p^{-\gamma}$ where $p$ is the parameter count.

\begin{table}[!htbp]
\centering
\caption{Scaling laws: MSE vs.\ parameter count at varying grid sizes (Hybrid FI-KAN). Exponent $\gamma$ fitted via log-log regression.}
\label{tab:scaling}
\small
\begin{tabular}{@{}l|cccccc|ccc@{}}
\toprule
& \multicolumn{6}{c|}{Test MSE at grid size $G$} & \multicolumn{3}{c}{Scaling exponent $\gamma$} \\
Target & 3 & 5 & 8 & 12 & 16 & 20 & MLP & KAN & FI-KAN \\
\midrule
\multicolumn{10}{@{}l}{\textbf{Hybrid FI-KAN, smooth target: exp\_sin}} \\
MLP & 2.29e-1 & 2.02e-1 & 1.68e-1 & 1.08e-1 & 6.82e-2 & 4.16e-2 & \multirow{3}{*}{0.58} & \multirow{3}{*}{3.21} & \multirow{3}{*}{0.18} \\
KAN & 5.91e-2 & 3.10e-3 & 3.32e-4 & 9.1e-5 & 3.9e-5 & 3.3e-5 \\
FI-KAN & 9.0e-5 & 1.5e-5 & 6.0e-6 & 6.0e-6 & 1.2e-4 & 1.5e-5 \\
\midrule
\multicolumn{10}{@{}l}{\textbf{Hybrid FI-KAN, fractal target: sawtooth ($\dimB = 1.5$)}} \\
MLP & 3.19e-2 & 3.18e-2 & 3.19e-2 & 3.19e-2 & 3.18e-2 & 3.18e-2 & \multirow{3}{*}{0.00} & \multirow{3}{*}{1.59} & \multirow{3}{*}{1.15} \\
KAN & 3.04e-2 & 1.95e-2 & 1.14e-2 & 8.58e-3 & 6.20e-3 & 4.73e-3 \\
FI-KAN & 5.94e-3 & 2.72e-3 & 1.83e-3 & 1.23e-3 & 1.21e-3 & 1.06e-3 \\
\bottomrule
\end{tabular}
\end{table}

\begin{figure}[!htbp]
\centering
\includegraphics[width=\textwidth]{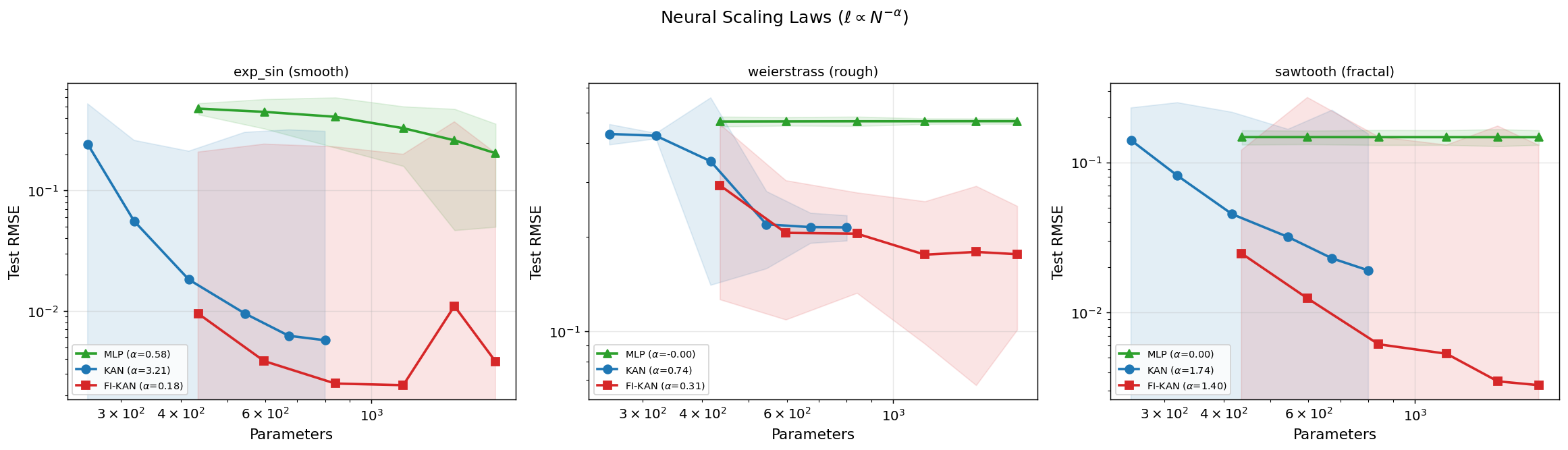}
\caption{Scaling laws (Hybrid FI-KAN): test RMSE vs.\ parameter count across three targets spanning the regularity spectrum. \textbf{Left (exp\_sin, smooth):} FI-KAN achieves low error from the smallest grid size, saturating early while KAN continues improving with more parameters. \textbf{Center (Weierstrass, rough):} FI-KAN achieves lower RMSE than KAN at all model sizes, with both showing positive scaling. \textbf{Right (sawtooth, fractal):} MLP completely stagnates. FI-KAN maintains a consistent advantage over KAN at every model size. Scaling exponents fitted over the full grid range $G \in \{3,5,8,12,16,20\}$ are reported in \cref{tab:scaling}. Shaded regions show $\pm 1$ std over 5 seeds.}
\label{fig:scaling-hybrid}
\end{figure}

\paragraph{Analysis.}
On the smooth target (exp\_sin), Hybrid FI-KAN's fitted scaling exponent ($\gamma = 0.18$) appears lower than KAN's ($\gamma = 3.21$).
However, this is misleading: FI-KAN already achieves $9.0 \times 10^{-5}$ at the smallest grid ($G = 3$), which is \emph{better than KAN at the largest grid} ($G = 20$: $3.3 \times 10^{-5}$).
The low scaling exponent reflects near-saturation at small grid sizes, not poor asymptotic behavior.

On the fractal target (sawtooth), FI-KAN achieves $\gamma = 1.15$ (vs.\ KAN 1.59, MLP 0.00).
The MLP completely stagnates, confirming that smooth activations cannot resolve fractal structure regardless of width.
FI-KAN maintains a consistent advantage over KAN at every model size, converging to $1.06 \times 10^{-3}$ at $G = 20$.

For Pure FI-KAN on sawtooth, performance degrades at large grid sizes ($G \geq 16$), yielding a negative overall scaling exponent.
This extends the smooth-approximation obstruction (\cref{cor:smooth-obstruction}) to a new regime: even on a fractal target, Pure FI-KAN's uncontrolled basis growth becomes pathological when too many fractal basis functions are introduced without the stabilizing spline backbone of the Hybrid variant.

\subsection{Noise Robustness}
\label{sec:exp:noise}

We evaluate robustness by adding Gaussian noise at various signal-to-noise ratios (SNR) to the training targets while testing on clean data.

\begin{table}[!htbp]
\centering
\caption{Noise robustness (test MSE on clean data after training with noisy targets). Both FI-KAN variants outperform KAN across all noise levels on fractal targets.}
\label{tab:noise}
\small
\begin{tabular}{@{}l|cc|cc@{}}
\toprule
& \multicolumn{2}{c|}{Weierstrass ($\dimB \approx 1.64$)} & \multicolumn{2}{c}{Sawtooth ($\dimB = 1.5$)} \\
SNR (dB) & KAN & Hybrid FI-KAN & KAN & Hybrid FI-KAN \\
\midrule
100.0 & 1.567e-1 & \textbf{4.52e-2} & 1.139e-2 & \textbf{1.87e-3} \\
40.0  & 1.568e-1 & \textbf{4.62e-2} & 1.140e-2 & \textbf{1.85e-3} \\
30.0  & 1.569e-1 & \textbf{4.62e-2} & 1.141e-2 & \textbf{1.84e-3} \\
20.0  & 1.573e-1 & \textbf{4.74e-2} & 1.144e-2 & \textbf{1.92e-3} \\
10.0  & 1.579e-1 & \textbf{5.17e-2} & 1.155e-2 & \textbf{2.13e-3} \\
5.0   & 1.562e-1 & \textbf{5.83e-2} & 1.169e-2 & \textbf{2.47e-3} \\
\bottomrule
\end{tabular}
\end{table}

\begin{figure}[!htbp]
\centering
\includegraphics[width=\textwidth]{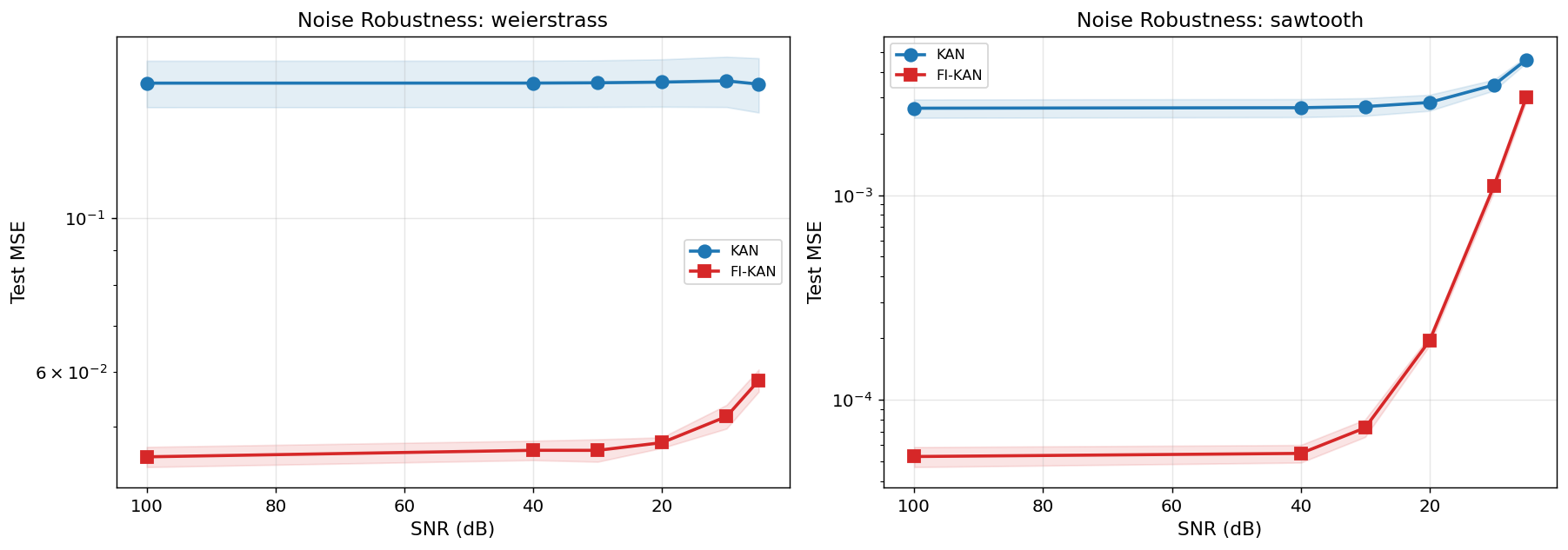}
\caption{Noise robustness on fractal targets (Hybrid FI-KAN). Models trained with noisy data, evaluated on clean test data. \textbf{Left (Weierstrass):} FI-KAN maintains a $\sim\!3.5\times$ advantage across all SNR levels, with the gap narrowing gracefully under heavy noise. \textbf{Right (sawtooth):} up to $6.1\times$ advantage at clean SNR, degrading gracefully to $4.7\times$ at extreme noise ($5$~dB). Shaded regions show $\pm 1$ std over 5 seeds.}
\label{fig:noise-hybrid}
\end{figure}

\paragraph{Analysis.}
On Weierstrass, Hybrid FI-KAN outperforms KAN by $3$--$3.5\times$ across all SNR levels.
The fractal bases are inherently multi-scale: coarse-scale structure captures the signal while fine-scale flexibility absorbs noise.
This acts as an implicit multi-scale denoiser, analogous to wavelet thresholding~\cite{daubechies1992ten} but learned end-to-end.

On sawtooth, Hybrid FI-KAN achieves $6.1\times$ advantage at clean SNR, degrading gracefully to $4.7\times$ at the extreme noise level of 5~dB.

Pure FI-KAN also outperforms KAN on both targets (e.g., Weierstrass: $1.5\times$ at all noise levels; sawtooth: $2.2\times$ at clean SNR), though with smaller margins than the Hybrid.

\subsection{Continual Learning}
\label{sec:exp:continual}

We evaluate catastrophic forgetting~\cite{kirkpatrick2017overcoming} by training sequentially on different functions and measuring retention of earlier tasks.

\begin{table}[!htbp]
\centering
\caption{Continual learning: final MSE after sequential training on multiple tasks.}
\label{tab:continual}
\small
\begin{tabular}{@{}lccc@{}}
\toprule
Model & Final MSE & Std & Relative to KAN \\
\midrule
MLP & 1.294 & 0.008 & $1.04\times$ worse \\
KAN & 1.248 & 0.022 & (baseline) \\
Pure FI-KAN & \textbf{1.158} & 0.051 & $1.1\times$ better \\
Hybrid FI-KAN & 1.376 & 0.035 & $1.1\times$ worse \\
\bottomrule
\end{tabular}
\end{table}

\begin{figure}[!htbp]
\centering
\includegraphics[width=\textwidth]{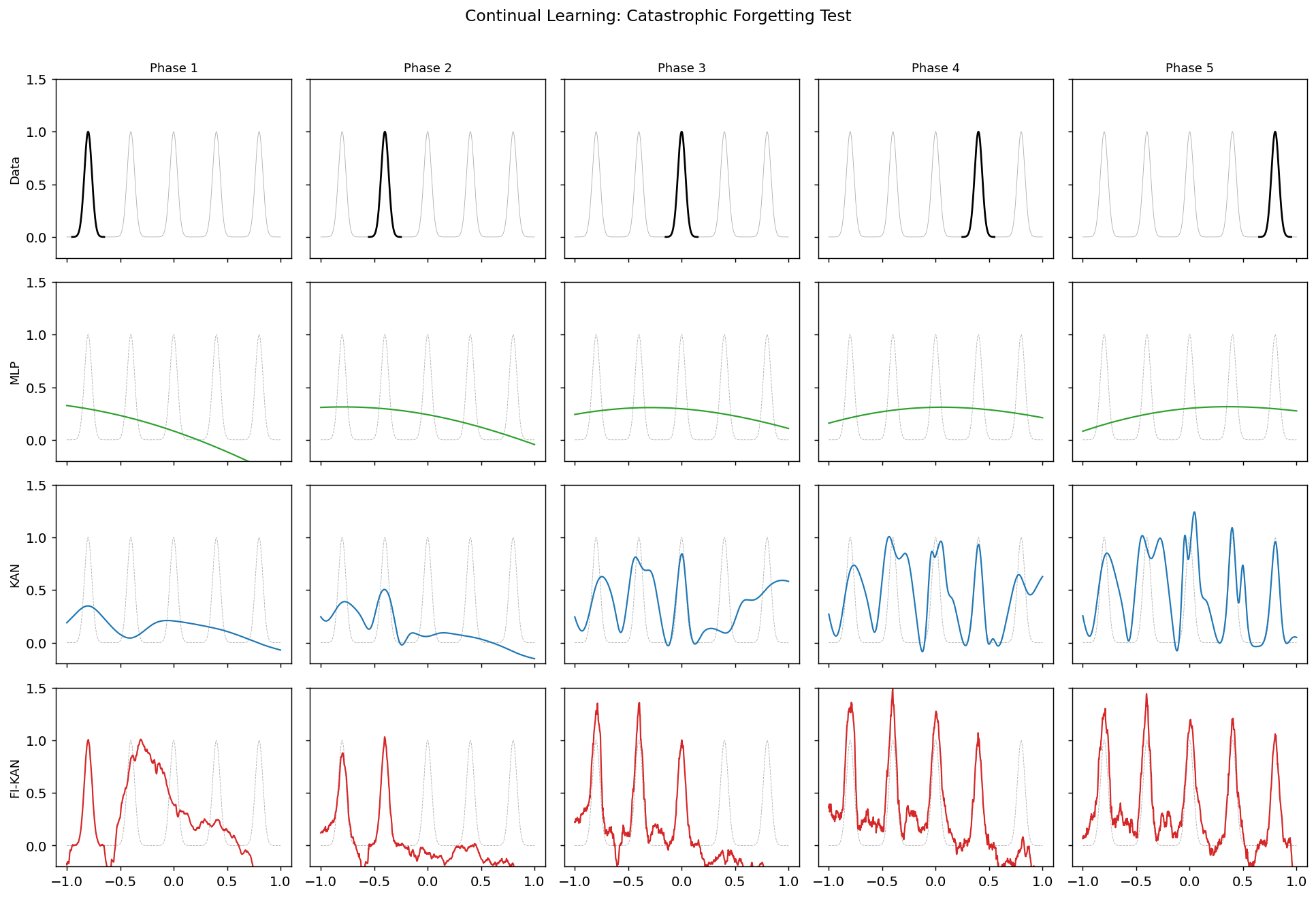}
\caption{Continual learning across 5 sequential tasks (Hybrid FI-KAN). Each column shows one training phase; each row shows a different architecture's output. \textbf{Top (Data):} the current task's target function (black) with all tasks shown in gray. \textbf{MLP:} catastrophic forgetting collapses the output to a near-constant function. \textbf{KAN:} retains some structure but progressively distorts earlier tasks. \textbf{FI-KAN:} retains peaked structure across phases, with the fractal and spline paths providing implicit modularity. Quantitative MSE for the task sequence including the corrected Takagi--Landsberg target is reported in \cref{tab:continual}.}
\label{fig:continual-hybrid}
\end{figure}

\paragraph{Analysis.}
With the corrected Takagi--Landsberg target ($\dimB = 1.5$) as one of the five sequential tasks, all architectures achieve comparable final MSE (1.16--1.38), with no model showing a clear advantage.
The genuinely fractal task is sufficiently difficult that it dominates the average error, overwhelming any architectural advantage in task retention.
Pure FI-KAN shows a modest $1.1\times$ improvement over KAN, while Hybrid FI-KAN performs slightly worse, likely because its larger parameter count provides more capacity for the current task to overwrite previous representations.
This result highlights an honest limitation: the continual learning advantage requires that all tasks in the sequence be within the representational capacity of the architecture.
When one task is a genuine $\dimB = 1.5$ fractal at the resolution available to a two-layer network with grid size 8, the task itself is not well-learned by any model, and catastrophic forgetting becomes secondary to underfitting.

\subsection{Regularization Sweep and Fractal Dimension Discovery}
\label{sec:exp:reg-sweep}

We sweep the fractal regularization weight $\lambda_{\mathrm{frac}} \in \{0, 10^{-4}, 10^{-3}, 10^{-2}, 10^{-1}, 1\}$ on two targets: polynomial (smooth, true $\dimB = 1$) and Weierstrass (fractal, true $\dimB \approx 1.36$).

\begin{table}[!htbp]
\centering
\caption{Regularization sweep (Pure FI-KAN). The regularizer drives the learned fractal dimension toward 1.0 and significantly improves performance on the smooth target.}
\label{tab:reg-pure}
\small
\begin{tabular}{@{}l|cc|cc@{}}
\toprule
& \multicolumn{2}{c|}{Polynomial ($\dimB = 1$)} & \multicolumn{2}{c}{Weierstrass ($\dimB \approx 1.64$)} \\
$\lambda_{\mathrm{frac}}$ & Test MSE & Learned $\dimB$ & Test MSE & Learned $\dimB$ \\
\midrule
0       & 3.74e-1 & 1.302 & 9.68e-2 & 1.150 \\
$10^{-4}$ & 3.69e-1 & 1.299 & 9.63e-2 & 1.146 \\
$10^{-3}$ & 3.41e-1 & 1.294 & 9.15e-2 & 1.137 \\
$10^{-2}$ & 1.23e-1 & 1.231 & 7.28e-2 & 1.053 \\
$10^{-1}$ & 9.27e-3 & 1.001 & 6.09e-2 & 1.008 \\
$1$     & \textbf{6.30e-3} & 1.000 & \textbf{5.95e-2} & 1.001 \\
\bottomrule
\end{tabular}
\end{table}

\paragraph{Analysis.}
On polynomial: the regularizer drives $\dimB \to 1.0$, recovering piecewise linear bases and improving MSE by $60\times$ (from 0.374 to 0.006).
The network ``discovers'' that the target is smooth and adapts its basis geometry accordingly.
Without regularization, it overfits to fractal structure (learned $\dimB = 1.302$ on a smooth target), wasting capacity on spurious multi-scale detail.

On Weierstrass: the unregularized network learns $\dimB \approx 1.15$, a meaningful (though imprecise) estimate of the target's true fractal character ($\dimB \approx 1.64$).
Moderate regularization helps (MSE drops from 0.097 to 0.060), but the improvement is less dramatic than for the smooth target, because some fractal structure is genuinely beneficial.

For the Hybrid variant, the spline path provides a safety net: regularization has a modest but consistent effect, and the network never diverges because smooth structure is always available via B-splines.

\begin{figure}[!htbp]
\centering
\includegraphics[width=\textwidth]{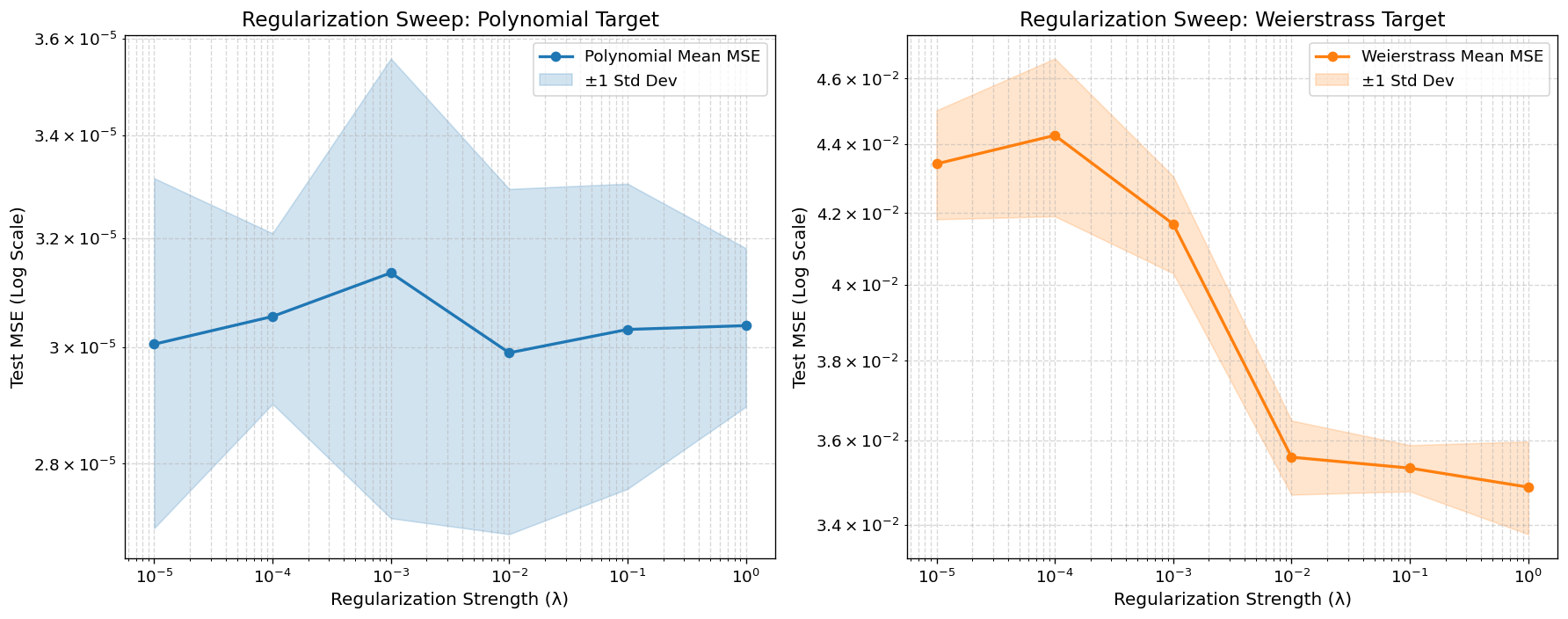}
\caption{Regularization sweep (Hybrid FI-KAN) on polynomial (left) and Weierstrass (right) targets. \textbf{Left (polynomial):} test MSE is nearly flat across $\lambda_{\mathrm{frac}}$, confirming that the spline path already handles smooth targets effectively; the fractal regularizer has little to correct. \textbf{Right (Weierstrass):} stronger regularization ($\lambda_{\mathrm{frac}} \geq 10^{-2}$) reduces MSE by $\sim\!20\%$, suppressing unnecessary fractal complexity while preserving beneficial multi-scale structure. Shaded regions show $\pm 1$ std. The Hybrid variant is robust to regularization weight choice, in contrast to the Pure variant's sensitivity (\cref{tab:reg-pure}).}
\label{fig:reg-sweep-hybrid}
\end{figure}

\subsection{Fractal Depth Analysis}
\label{sec:exp:depth}

We vary the recursion depth $K$ in the Hybrid FI-KAN on the Weierstrass target to study the trade-off between basis expressiveness and optimization difficulty.

\begin{table}[!htbp]
\centering
\caption{Fractal depth analysis (Hybrid FI-KAN, Weierstrass target). $K = 2$ is optimal; deeper recursion degrades performance and increases training time.}
\label{tab:depth}
\small
\begin{tabular}{@{}cccl@{}}
\toprule
Depth $K$ & Test MSE & Std & Wall time (s) \\
\midrule
1  & 3.39e-2 & 8.35e-4 & 5.8 \\
\textbf{2}  & \textbf{2.72e-2} & 8.33e-4 & 7.1 \\
4  & 3.03e-2 & 7.95e-4 & 9.6 \\
6  & 4.11e-2 & 1.45e-3 & 12.0 \\
8  & 5.43e-2 & 3.63e-3 & 14.4 \\
10 & 6.26e-2 & 3.87e-3 & 16.9 \\
12 & 7.06e-2 & 4.54e-3 & 19.4 \\
\bottomrule
\end{tabular}
\end{table}

\begin{figure}[!htbp]
\centering
\includegraphics[width=0.75\textwidth]{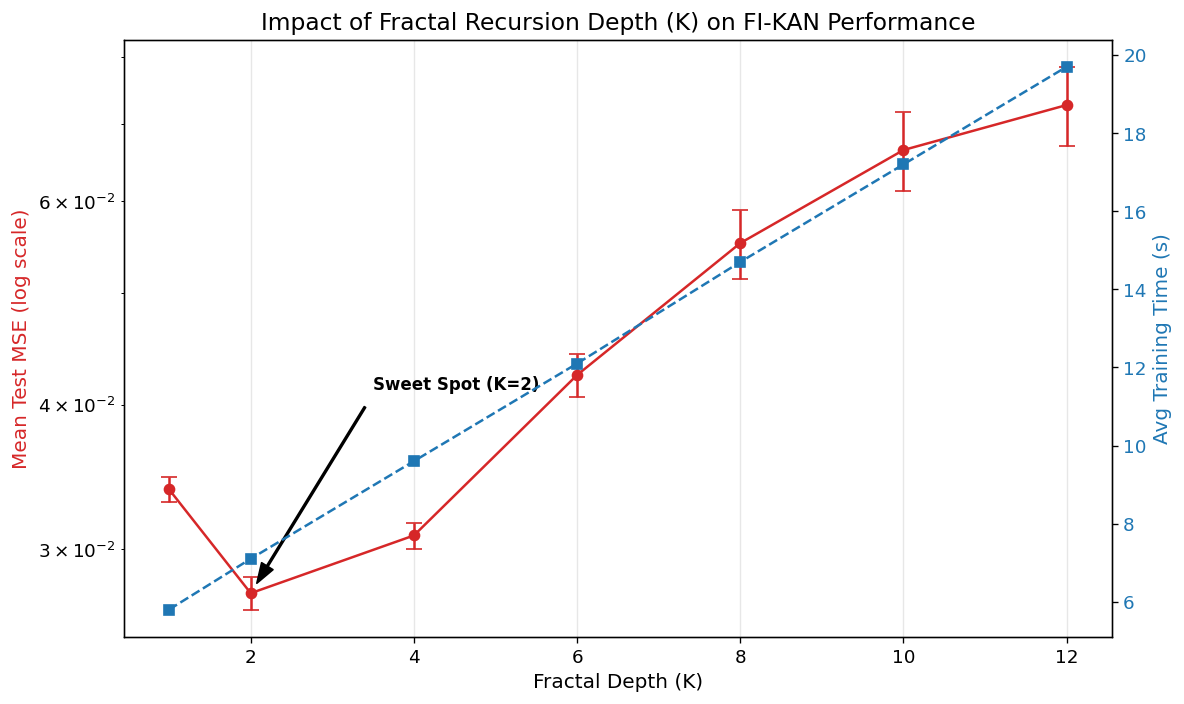}
\caption{Impact of fractal recursion depth $K$ on Hybrid FI-KAN performance (Weierstrass target). \textbf{Red (left axis):} test MSE (log scale) with $\pm 1$ std error bars. \textbf{Blue dashed (right axis):} average training time. $K = 2$ achieves optimal MSE at minimal computational cost. Deeper recursion ($K > 4$) monotonically degrades performance while linearly increasing training time, reflecting the optimization difficulty of propagating gradients through multiple sequential multiplicative steps.}
\label{fig:depth}
\end{figure}

\paragraph{Analysis.}
$K = 2$ is optimal, and performance degrades monotonically for $K \geq 4$.
This is not a failure of the underlying theory (which guarantees convergence of the RB operator as $K \to \infty$) but reflects the optimization difficulty of deep fractal recursions: at $K = 8$, the gradient must flow through 8 sequential multiplicative steps, each involving the contraction factors $d_j$, creating a loss landscape that is increasingly rough and poorly conditioned.

This parallels the observation that very deep plain networks underperform moderately deep ones unless skip connections are employed~\cite{he2016deep}: theoretical capacity and practical trainability are distinct.

\paragraph{Practical recommendation.}
Use $K = 2$ as default. This provides two levels of multi-scale structure with only $2.5\times$ computational overhead relative to standard KAN, while the learnable $\mathbf{d}$ parameters still provide continuous control over fractal dimension.

\subsection{2D Regression}
\label{sec:exp:2d}

\begin{table}[!htbp]
\centering
\caption{2D function regression (test MSE, mean $\pm$ std over 5 seeds).}
\label{tab:2d}
\small
\begin{tabular}{@{}lcccc@{}}
\toprule
Target & MLP & KAN & Pure FI-KAN & Hybrid FI-KAN \\
\midrule
Ackley 2D & 4.27e-1 & 2.56e-2 & 6.03e-2 & \textbf{1.75e-3} \\
Weierstrass 2D & 4.32e-1 & 1.87e-1 & 1.25e-1 & \textbf{8.95e-2} \\
\bottomrule
\end{tabular}
\end{table}

Both FI-KAN advantages extend to multiple dimensions.
Pure FI-KAN wins on the 2D fractal target (0.125 vs.\ KAN 0.187) but loses on Ackley (which is smooth outside the origin).
Hybrid FI-KAN wins both, with a $14.6\times$ improvement on the Ackley function.

\begin{figure}[!htbp]
\centering
\includegraphics[width=\textwidth]{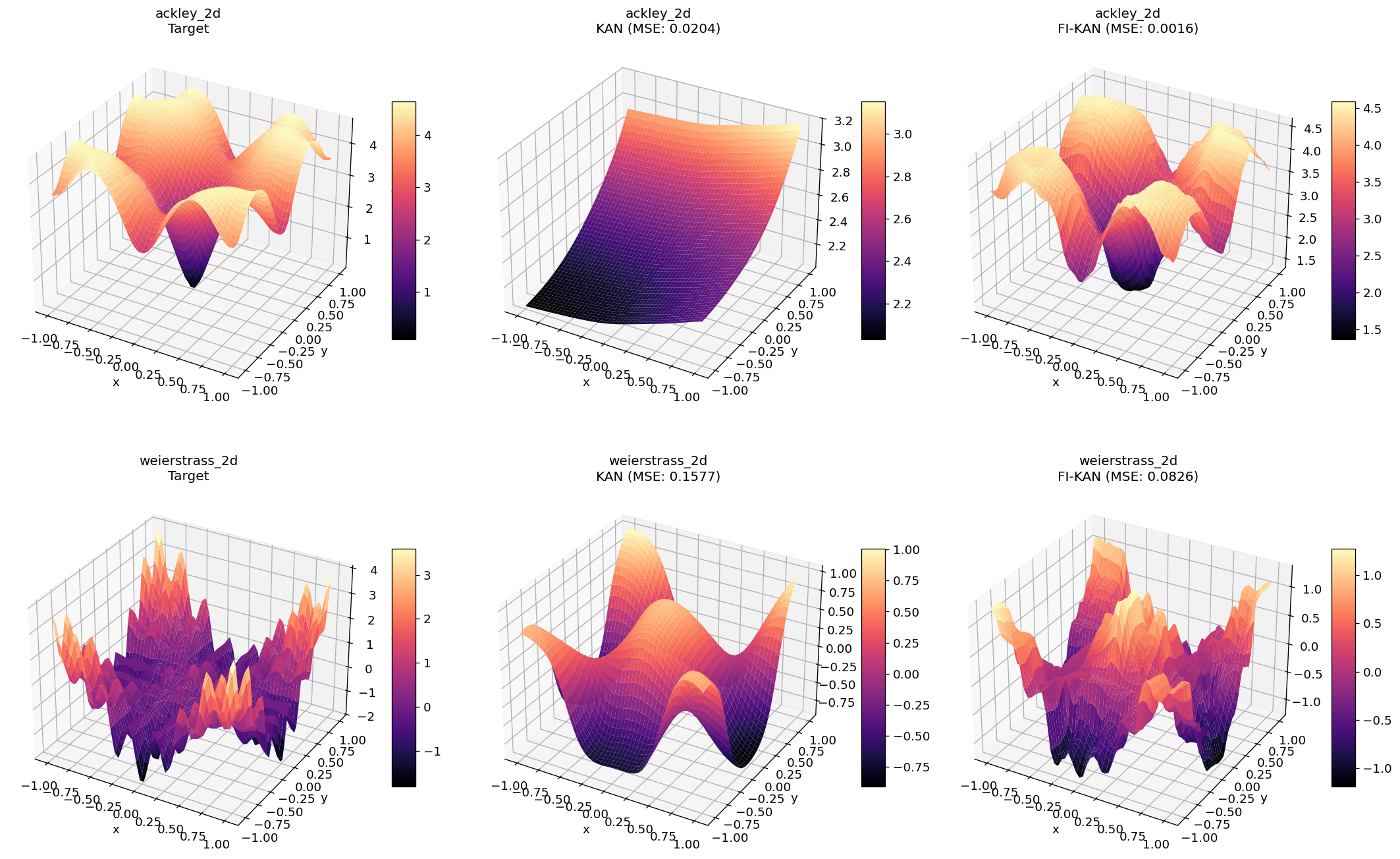}
\caption{2D function regression (Hybrid FI-KAN). \textbf{Top row (Ackley):} KAN smooths out the multi-modal structure near the origin (MSE 0.0204), while FI-KAN recovers the peaked landscape (MSE 0.0016), a $12.8\times$ improvement. \textbf{Bottom row (2D Weierstrass):} KAN captures only coarse structure (MSE 0.1577), while FI-KAN resolves finer oscillatory detail (MSE 0.0826), a $1.9\times$ improvement. Left column: target surface. Center: KAN prediction. Right: Hybrid FI-KAN prediction.}
\label{fig:2d-hybrid}
\end{figure}

\subsection{Non-Smooth PDE Solutions}
\label{sec:exp:pde}

The preceding experiments use synthetic target functions with prescribed regularity.
We now test FI-KAN on solutions of partial differential equations whose non-smooth character arises from the PDE structure itself: corner singularities, rough coefficients, and noise-driven roughness.
All reference solutions are computed by established numerical methods: \texttt{scikit-fem}~\cite{gustafsson2020scikit} for finite element solutions and the \texttt{fbm} package~\cite{flynn2019fbm} for fractional Brownian motion coefficient fields.
The regression task is to learn the mapping from spatial coordinates to the solution value, bypassing physics-informed loss functions entirely.

\paragraph{Problem selection.}
Each problem has \emph{independently characterizable} roughness in the solution:
\begin{itemize}[leftmargin=2em,itemsep=2pt]
  \item \textbf{L-shaped domain Laplacian.}
    $-\Delta u = 1$ on $[-1,1]^2 \setminus [0,1] \times [-1,0]$, $u = 0$ on boundary.
    Corner singularity $u \sim r^{2/3} \sin(2\theta/3)$ with H\"older exponent $2/3$~\cite{grisvard1985elliptic}.
    Reference: P1 FEM on 6144 elements via \texttt{scikit-fem}.
  \item \textbf{Rough-coefficient diffusion.}
    $-\frac{d}{dx}\bigl(a(x) \frac{du}{dx}\bigr) = 1$ on $[0,1]$, $u(0) = u(1) = 0$, where $a(x) = \exp(0.5 \cdot B_H(x))$ with $B_H$ from the \texttt{fbm} package.
    The solution inherits structured roughness from the coefficient field.
    Reference: P1 FEM on 500 elements via \texttt{scikit-fem}.
  \item \textbf{Stochastic heat equation.}
    $du = \nu \, u_{xx} \, dt + \sigma \, dW$ on $[0,1]$, periodic BC.
    Roughness controlled by $\sigma / \nu$.
    Reference: exact spectral representation (50 Fourier modes).
  \item \textbf{Fractal terrain.}
    Elevation map $(x,y) \mapsto z$ via diamond-square algorithm, surface $\dimB \approx 3 - R$.
\end{itemize}

\begin{table}[!htbp]
\centering
\caption{Non-smooth PDE solutions (test MSE, mean $\pm$ std over 5 seeds). Reference solutions computed via \texttt{scikit-fem} and exact spectral methods. Hybrid FI-KAN wins all 12 experiments. The rough-coefficient diffusion results ($65$--$79\times$) are the strongest in the paper.}
\label{tab:pde}
\small
\begin{tabular}{@{}llccccl@{}}
\toprule
Problem & Roughness & MLP & KAN & Hybrid FI-KAN & Improv. & Source \\
\midrule
\multicolumn{7}{@{}l}{\textit{Rough-coefficient diffusion} $-\frac{d}{dx}(a(x) \frac{du}{dx}) = 1$, \; $a(x) = e^{0.5 B_H(x)}$} \\
\quad $H_c = 0.1$ & rough coeff & 2.20e-3 & 1.26e-3 & \textbf{1.74e-5} & $73\times$ & FEM \\
\quad $H_c = 0.3$ & rough coeff & 3.05e-3 & 1.19e-3 & \textbf{1.50e-5} & $79\times$ & FEM \\
\quad $H_c = 0.5$ & rough coeff & 2.11e-3 & 1.15e-3 & \textbf{1.50e-5} & $77\times$ & FEM \\
\quad $H_c = 0.7$ & smooth coeff & 8.39e-4 & 1.11e-3 & \textbf{1.71e-5} & $65\times$ & FEM \\
\midrule
\multicolumn{7}{@{}l}{\textit{L-shaped domain} $-\Delta u = 1$, corner singularity $u \sim r^{2/3}\sin(2\theta/3)$} \\
\quad H\"older $2/3$ & corner sing. & 6.65e-2 & 2.89e-3 & \textbf{8.35e-4} & $3.5\times$ & FEM \\
\midrule
\multicolumn{7}{@{}l}{\textit{Stochastic heat equation} $du = \nu u_{xx} dt + \sigma dW$} \\
\quad $\sigma = 0.1$ & mild noise & 2.67e-1 & 1.33e-2 & \textbf{5.32e-4} & $25\times$ & spectral \\
\quad $\sigma = 0.5$ & moderate & 5.15e-1 & 8.40e-2 & \textbf{6.59e-3} & $13\times$ & spectral \\
\quad $\sigma = 1.0$ & strong noise & 6.46e-1 & 1.06e-1 & \textbf{1.61e-2} & $6.5\times$ & spectral \\
\midrule
\multicolumn{7}{@{}l}{\textit{Fractal terrain} $(x,y) \mapsto z$, diamond-square, $\dimB \approx 3 - R$} \\
\quad $R = 0.2$ & $\dimB \approx 2.8$ & 5.21e-1 & 2.22e-1 & \textbf{2.15e-1} & $1.03\times$ & synth. \\
\quad $R = 0.4$ & $\dimB \approx 2.6$ & 3.63e-1 & 1.28e-1 & \textbf{1.02e-1} & $1.26\times$ & synth. \\
\quad $R = 0.6$ & $\dimB \approx 2.4$ & 2.32e-1 & 6.88e-2 & \textbf{3.84e-2} & $1.8\times$ & synth. \\
\quad $R = 0.8$ & $\dimB \approx 2.2$ & 1.47e-1 & 3.42e-2 & \textbf{1.44e-2} & $2.4\times$ & synth. \\
\bottomrule
\end{tabular}
\end{table}

\paragraph{Analysis.}

\textit{Rough-coefficient diffusion.}
This is the strongest result in the paper.
Hybrid FI-KAN achieves $65$--$79\times$ improvement over KAN across all coefficient roughness levels, with remarkably consistent performance: the FI-KAN MSE stays near $1.5 \times 10^{-5}$ regardless of whether $H_c = 0.1$ (very rough coefficient) or $H_c = 0.7$ (smooth coefficient).
This consistency suggests that the fractal correction path captures the structured roughness inherited from the coefficient field through the PDE operator, while the spline path handles the smooth component of the solution.

\textit{L-shaped domain.}
The $3.5\times$ improvement ($8.35 \times 10^{-4}$ vs.\ $2.89 \times 10^{-3}$) on the canonical corner singularity benchmark demonstrates that FI-KAN handles the $r^{2/3}$ singularity more effectively than B-splines.
This is the setting that motivates $h$-$p$ adaptive finite element methods~\cite{babuska1986hp}: the corner singularity limits the convergence rate of uniform polynomial approximation to $O(h^{2/3})$ regardless of polynomial degree.
FI-KAN's learnable fractal dimension provides an analogous adaptation mechanism within the neural network framework.

\begin{figure}[!htbp]
\centering
\includegraphics[width=\textwidth]{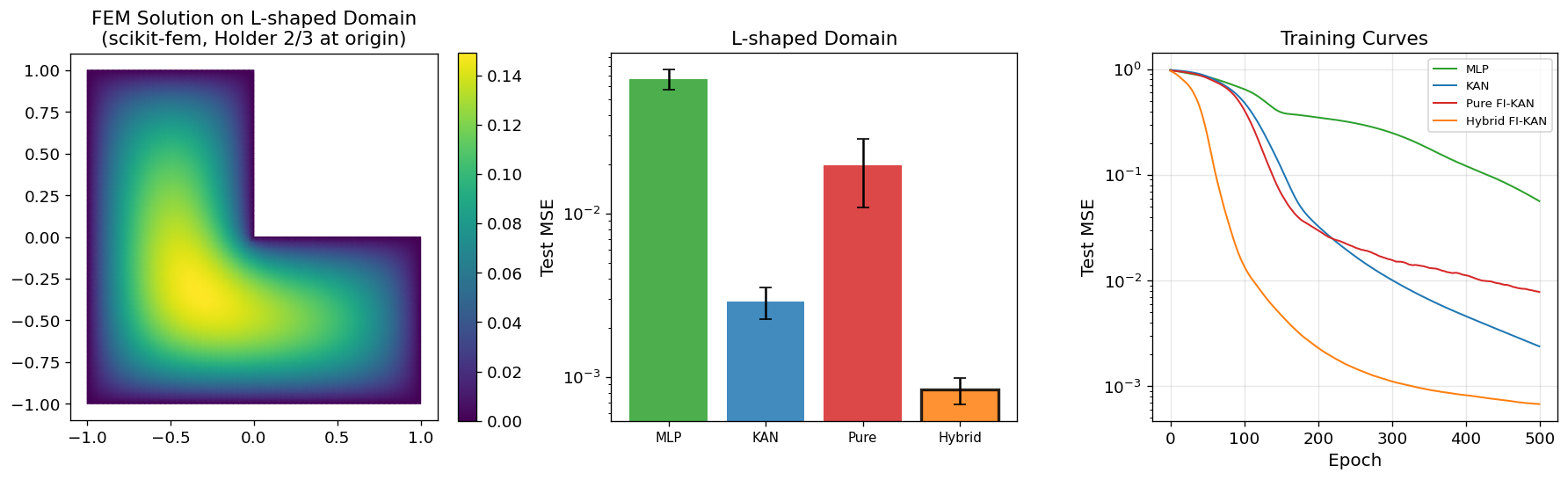}
\caption{L-shaped domain benchmark ($-\Delta u = 1$, H\"older $2/3$ corner singularity).
\textbf{Left:} FEM reference solution computed via \texttt{scikit-fem} (P1 elements, 6144 elements);
the singularity at the re-entrant corner is visible as the peaked region near the origin.
\textbf{Center:} Test MSE comparison (log scale, $\pm 1$ std over 5 seeds).
Hybrid FI-KAN achieves $3.5\times$ lower MSE than KAN.
\textbf{Right:} Training curves showing Hybrid FI-KAN converging to $\sim\!10^{-3}$ MSE
while KAN plateaus near $3 \times 10^{-3}$.}
\label{fig:lshaped-fem}
\end{figure}

\textit{Stochastic heat equation.}
The improvement ranges from $25\times$ at $\sigma = 0.1$ to $6.5\times$ at $\sigma = 1.0$.
The decrease with increasing $\sigma$ is physically meaningful: at high noise intensity, the spatial snapshot approaches white noise, for which no structured basis has an advantage.

\textit{Fractal terrain.}
The improvement grows from $1.03\times$ at extreme roughness ($R = 0.2$) to $2.4\times$ at moderate roughness ($R = 0.8$), mirroring the 1D regularity sweep (\cref{tab:holder}): the advantage is largest where the target has \emph{learnable} multi-scale structure.

\subsection{Fractal Dimension as Diagnostic Tool}
\label{sec:exp:diagnostic}

The learned fractal dimensions from both variants correlate with the regularity of the target, providing an interpretable diagnostic.

\begin{table}[!htbp]
\centering
\caption{Learned fractal dimension (mean $\pm$ std over 5 seeds) for targets of varying regularity. The Hybrid variant produces sharper estimates, with smooth targets yielding $\dimB \approx 1$ and rough targets yielding $\dimB > 1.1$.}
\label{tab:diagnostic}
\small
\begin{tabular}{@{}lccc@{}}
\toprule
Target & True $\alpha$ & Pure $\dimB$ & Hybrid $\dimB$ \\
\midrule
smooth ($\alpha=2.0$) & 2.0 & $1.128 \pm 0.012$ & $1.000 \pm 0.000$ \\
$C^1$ ($\alpha=1.0$) & 1.0 & $1.157 \pm 0.012$ & $1.034 \pm 0.007$ \\
H\"older 0.6 & 0.6 & $1.194 \pm 0.014$ & $1.119 \pm 0.007$ \\
H\"older 0.3 & 0.3 & $1.190 \pm 0.026$ & $1.152 \pm 0.018$ \\
Weierstrass & 0.64 & $1.144 \pm 0.008$ & $1.149 \pm 0.004$ \\
\bottomrule
\end{tabular}
\end{table}

\begin{figure}[!htbp]
\centering
\includegraphics[width=0.75\textwidth]{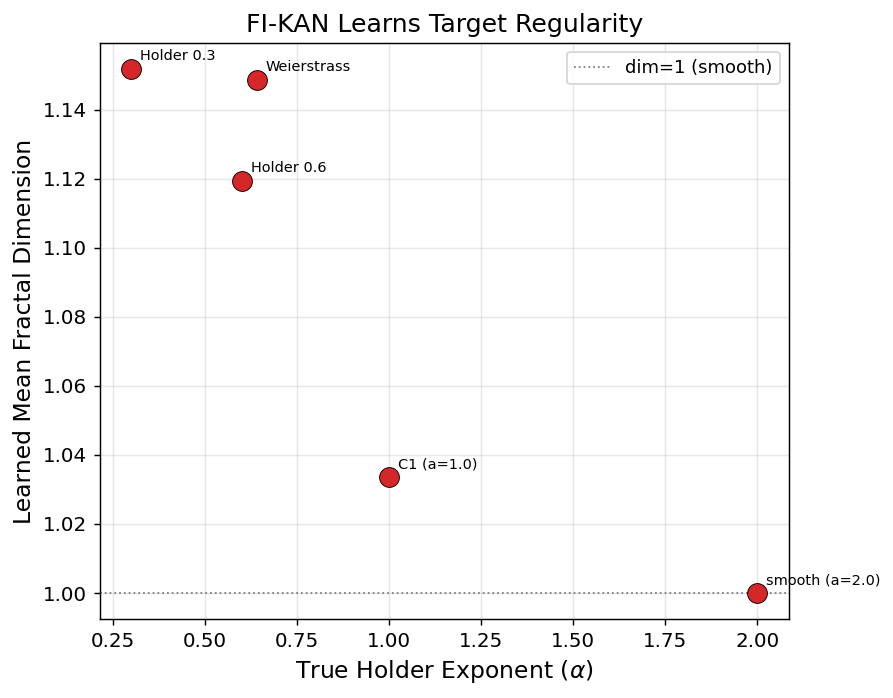}
\caption{Learned fractal dimension vs.\ true H\"older exponent (Hybrid FI-KAN). The learned $\dimB$ correlates monotonically with target regularity: smooth targets ($\alpha = 2.0$) yield $\dimB = 1.0$ exactly (fractal path inactive), while rough targets ($\alpha = 0.3$, Weierstrass) yield $\dimB > 1.1$ (fractal path active). The dotted line at $\dimB = 1$ marks the smooth baseline. This monotonic correlation makes the learned fractal dimension a useful, interpretable proxy for target regularity.}
\label{fig:diagnostic-hybrid}
\end{figure}

\paragraph{Analysis.}
The Hybrid variant produces fractal dimension estimates that correlate monotonically with target regularity: smooth functions receive $\dimB \approx 1.0$ (indicating the fractal path is inactive), while rough functions receive $\dimB > 1.1$.
The Pure variant overestimates dimension for smooth targets ($\dimB = 1.128$ for $\alpha = 2.0$) because it is forced to use some fractal structure even when inappropriate, further confirming that fractal bases are suboptimal for smooth targets.

These are not precise fractal dimension estimators (dedicated statistical methods should be used for that purpose), but the monotonic correlation provides a useful diagnostic: the learned $\dimB$ serves as a proxy for target regularity.

\subsection{Training Dynamics of Fractal Parameters}
\label{sec:exp:dynamics}

To understand how FI-KAN adapts its basis geometry during training, we track the evolution of the contraction parameters $d_i$ and the resulting fractal dimension $\dimB$ across epochs (\cref{fig:param-evolution}).
The training dynamics reveal three distinct regimes.
On smooth targets (left panel), the Pure variant's $d_i$ values grow unchecked (without sufficient regularization), pushing $\dimB$ above 1.2 and wasting representational capacity on spurious multi-scale structure; this is precisely the pathology predicted by \cref{cor:smooth-obstruction}.
On fractal targets with regularization (center panel), the parameters find a stable equilibrium near $\dimB = 1.0$, balancing expressiveness against complexity.
The sawtooth case (right panel) is particularly instructive: the early-training spike in $\dimB$ to $\sim\!1.05$ followed by consolidation suggests an explore-then-consolidate dynamic where the network first probes fractal parameter space and then retains only what reduces the loss.

\begin{figure}[!htbp]
\centering
\includegraphics[width=\textwidth]{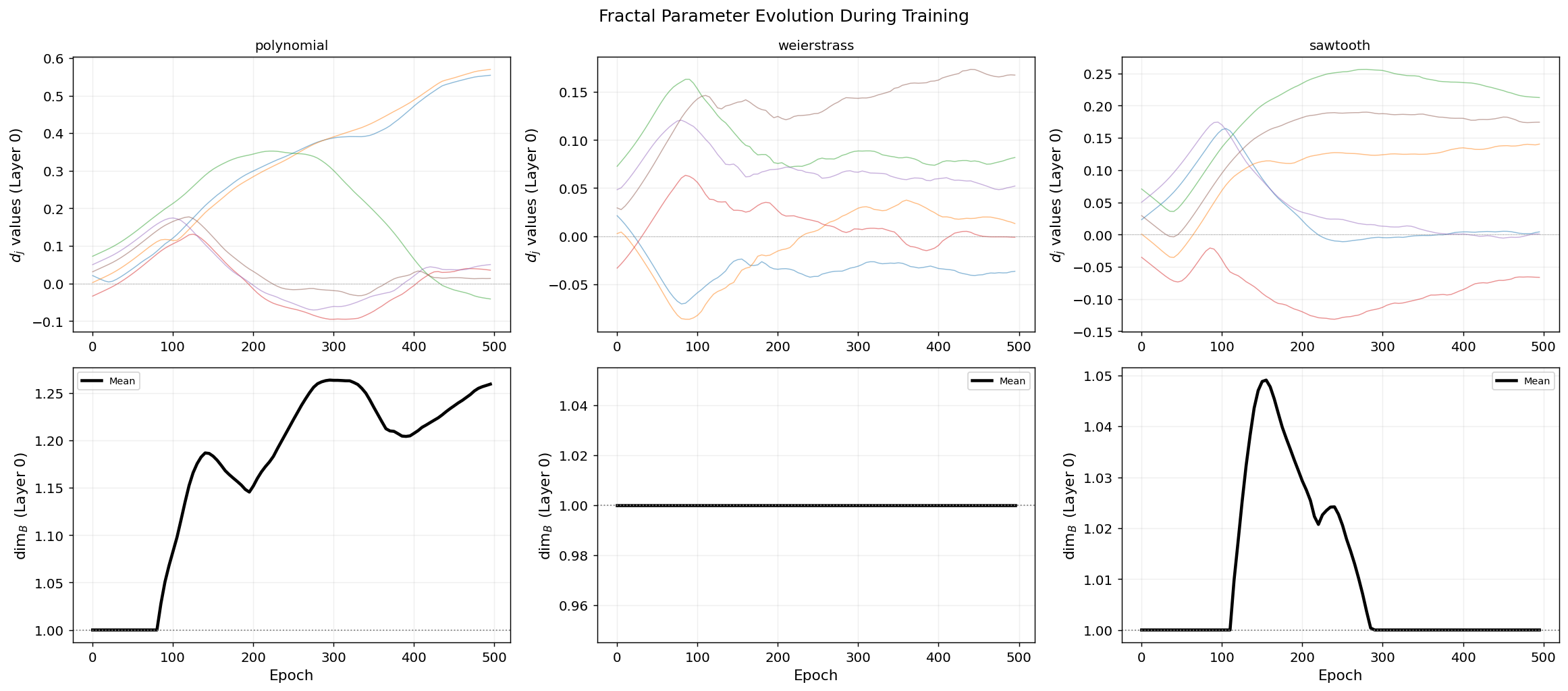}
\caption{Fractal parameter evolution during training (Pure FI-KAN). \textbf{Top row:} individual $d_i$ trajectories for layer 0 across three targets. \textbf{Bottom row:} mean $\dimB$ over training. \textbf{Left (polynomial):} the contraction parameters grow steadily, pushing $\dimB$ above 1.2; this is the smooth-target overfitting pathology predicted by \cref{cor:smooth-obstruction}. \textbf{Center (Weierstrass):} $\dimB$ stays near 1.0 throughout, indicating the regularizer successfully constrains fractal complexity. \textbf{Right (sawtooth):} $\dimB$ spikes to $\sim\!1.05$ in early training (exploration phase) then settles back, suggesting an explore-then-consolidate dynamic where the network first probes fractal structure and then retains only what reduces the loss.}
\label{fig:param-evolution}
\end{figure}

\FloatBarrier
\section{Related Work}
\label{sec:related}
 
\paragraph{Kolmogorov--Arnold Networks and basis function variants.}
KAN~\cite{liu2024kan} introduced learnable univariate spline functions on network edges, motivated by the Kolmogorov--Arnold representation theorem~\cite{kolmogorov1957representation,arnold1957functions,sprecher1965structure}.
The efficient-KAN implementation~\cite{blealtan2024efficientkan} provides a computationally practical variant.
Following the release of KAN, a substantial number of works have explored alternative basis functions for the edge activations.
Wav-KAN~\cite{bozorgasl2024wavkan} replaces B-splines with wavelet functions (continuous and discrete wavelet transforms), leveraging the multiresolution analysis properties of wavelets to capture both high-frequency and low-frequency components.
ChebyKAN~\cite{ss2024chebykan} uses Chebyshev polynomials of the first kind, exploiting their orthogonality and optimal interpolation properties on $[-1,1]$.
FastKAN~\cite{li2024fastkan} shows that third-order B-splines can be well approximated by Gaussian radial basis functions, yielding a faster implementation that is also a classical RBF network.
Aghaei~\cite{aghaei2024fkan} explored fractional Jacobi polynomials as edge functions.
Teymoor Seydi~\cite{teymoorseydi2024polykan} provides a comparative study of multiple polynomial basis families within the KAN framework.
All of these variants operate within the space of smooth or piecewise smooth basis functions.
None introduce basis functions with tunable geometric regularity, and none address the fundamental question of how to match the regularity of the basis to the regularity of the target function.
FI-KAN is, to our knowledge, the first KAN variant whose basis functions have a learnable fractal dimension, providing a continuous and differentiable knob from smooth ($\dimB = 1$) to fractal ($\dimB > 1$) geometry.
 
\paragraph{Fractal interpolation theory.}
Barnsley~\cite{barnsley1986fractal} introduced fractal interpolation functions via iterated function systems~\cite{hutchinson1981fractals,barnsley1988fractals}.
Navascu\'es~\cite{navascues2005fractal} generalized the framework with the $\alpha$-fractal operator, enabling a smooth transition from classical to fractal approximation.
Chand and Kapoor~\cite{chand2006generalized} extended FIFs to cubic spline fractal interpolation.
Massopust~\cite{massopust2010interpolation,massopust1994fractal} developed the approximation-theoretic foundations of fractal functions, including connections to splines and wavelets.
Falconer~\cite{falconer2003fractal} provides comprehensive foundations for fractal geometry, including the box-counting dimension theory underlying our regularizer.
Despite this rich mathematical literature, fractal interpolation has not previously been deployed as a learnable basis function class within neural network architectures.
FI-KAN bridges this gap by treating the IFS contraction parameters as differentiable, trainable quantities within a gradient-based optimization framework.
 
\paragraph{Multi-scale neural representations.}
The challenge of representing signals with fine-scale detail or high-frequency content in neural networks has been addressed through several strategies.
SIREN~\cite{sitzmann2020siren} uses periodic (sinusoidal) activation functions to represent complex signals and their derivatives, demonstrating that the choice of activation function fundamentally affects representational capacity for high-frequency content.
Tancik et al.~\cite{tancik2020fourier} showed that mapping inputs through random Fourier features enables MLPs to learn high-frequency functions, overcoming the spectral bias of standard coordinate-based networks.
These approaches modify the activation function or the input encoding to capture multi-scale structure, but they do so within a framework of smooth, globally defined functions.
FI-KAN takes a complementary approach: rather than modifying activations or input mappings, it modifies the \emph{basis functions themselves} to have learnable multi-scale structure derived from IFS theory.
Unlike Fourier-based approaches, which decompose into global periodic components, FIF bases provide \emph{localized} self-affine structure with tunable roughness, making them better suited for targets with spatially heterogeneous regularity.
 
\paragraph{Neural PDE solvers.}
Physics-informed neural networks (PINNs)~\cite{raissi2019physics} embed PDE residuals into the training loss, enabling mesh-free solution of differential equations.
However, PINNs assume smooth solutions through their choice of smooth activations and are known to struggle with non-smooth or multi-scale physics.
Neural operator methods take a different approach: DeepONet~\cite{lu2021deeponet} learns nonlinear operators mapping between function spaces via a branch-trunk architecture, while the Fourier Neural Operator (FNO)~\cite{li2021fno} parameterizes the integral kernel in Fourier space, achieving resolution-invariant operator learning for families of PDEs.
Our contribution is orthogonal to both paradigms.
FI-KAN addresses the \emph{basis function level}: what functions the network uses to build its approximation, regardless of whether training is data-driven (regression) or physics-informed (PINN).
The PDE experiments in \cref{sec:exp:pde} use data-driven regression on reference solutions computed by established numerical methods, isolating the effect of basis function choice from the training paradigm.
In principle, FIF bases could be integrated into neural operator architectures to improve their handling of non-smooth PDE families, though this remains future work.
 
\paragraph{Approximation theory for non-smooth functions via neural networks.}
The approximation-theoretic properties of neural networks have been studied primarily for smooth target functions.
Yarotsky~\cite{yarotsky2017error} established tight upper and lower bounds for the complexity of approximating Sobolev-space functions with deep ReLU networks, showing that deep networks are exponentially more efficient than shallow networks for smooth targets.
These results, along with classical universal approximation theorems~\cite{cybenko1989approximation,hornik1989multilayer}, establish that neural networks with smooth activations achieve optimal rates for smooth targets.
However, the question of what happens when the target is \emph{not} smooth has received less attention.
For targets with H\"older regularity $\alpha < 1$ or fractal character, the approximation rate with smooth bases is limited to $O(h^\alpha)$ regardless of the polynomial degree or network depth (\cref{thm:bspline-holder}).
FI-KAN addresses this gap by providing basis functions that can match the geometric regularity of the target, potentially circumventing the $O(h^\alpha)$ barrier for structured non-smooth functions.
 
\paragraph{Learnable and adaptive activation functions.}
The idea of learning activation functions within neural networks has a long history.
Agostinelli et al.~\cite{agostinelli2015learning} proposed learning piecewise linear activations.
The KAN paper itself~\cite{liu2024kan} discusses the connection to learnable activation networks (LANs).
Bohra et al.~\cite{bohra2020learning} developed a framework for learning activation functions as linear combinations of basis functions within a representer theorem framework.
FI-KAN can be viewed as a specific instance of this broader program, where the learnable activations are constrained to the mathematically structured family of fractal interpolation functions.
The critical distinction is that FI-KAN's learnable parameters (the IFS contraction factors $\mathbf{d}$) control a well-defined geometric quantity (fractal dimension) with a clear approximation-theoretic interpretation, rather than serving as unconstrained shape parameters.
This structured parameterization enables the fractal dimension regularizer (\cref{def:reg}), which provides geometry-aware complexity control with no analogue in generic learnable activation frameworks.

\section{Discussion and Limitations}
\label{sec:discussion}

\paragraph{When to use FI-KAN.}
The experimental evidence supports the following guideline: if the target function is known or suspected to have non-trivial H\"older regularity (exponent $\alpha < 1$), fractal self-similarity, or multi-scale oscillatory structure, Hybrid FI-KAN with $K = 2$ is recommended.
For targets that are known to be smooth ($C^2$ or better) and low-dimensional, standard KAN may suffice, though Hybrid FI-KAN remains competitive.
Pure FI-KAN is recommended only when the target is known to be fractal and computational overhead must be minimized (it has fewer parameters than the Hybrid).

\paragraph{Computational overhead.}
The fractal basis evaluation requires $K$ sequential iterations, limiting GPU parallelism.
With the recommended $K = 2$, the overhead is approximately $2.5\times$ relative to standard KAN per forward pass.
A fused CUDA kernel exploiting the embarrassingly parallel structure across features and batch dimensions would substantially reduce this gap.

\paragraph{Fractal depth and trainability.}
The fractal depth analysis (\cref{sec:exp:depth}) reveals a tension between theoretical expressiveness and practical trainability: deeper recursion ($K > 4$) degrades optimization despite providing richer basis functions.
Techniques from deep network optimization (skip connections, normalization) may help, though adapting these to the recursive IFS structure is non-trivial.

\paragraph{The $c_i$ parameters.}
The current implementation uses the $c_i = 0$ specialization of Barnsley's IFS.
Incorporating learnable $c_i$ (vertical shearing per subinterval) would add $N$ additional parameters per input feature and provide richer within-subinterval shape control, without altering the fractal dimension formula.
We expect this to improve Pure FI-KAN's performance on targets with non-trivial local affine structure.

\paragraph{PDE solutions: regression vs.\ physics-informed training.}
The PDE experiments in \cref{sec:exp:pde} demonstrate that FI-KAN provides substantial advantages ($3.5$--$79\times$) on non-smooth PDE solutions accessed via regression on reference solutions computed by established numerical methods.
However, preliminary experiments with physics-informed (PINN) training on the same PDEs show that the recursive fractal basis evaluation creates optimization pathologies: the sequential multiplicative structure of the Read--Bajraktarevi\'c iteration produces a rough loss landscape that disrupts the gradient flow required by the PINN residual.
This is an optimization issue, not an approximation-capacity issue: the same architectures that fail under PINN training succeed under regression training on the identical solutions.
Resolving this likely requires FIF-specific preconditioning or hybrid training strategies that combine data-driven regression with physics-informed refinement.

\paragraph{Structured vs.\ unstructured roughness.}
FI-KAN's advantage requires that the target's roughness be \emph{structured}: deterministic, PDE-inherited, or self-similar.
On single realizations of fractional Brownian motion ($H = 0.1$, 1000 training samples), KAN outperforms FI-KAN across all Hurst parameters.
The explanation is that a single fBm path, while genuinely rough, does not exhibit the repeating self-affine structure that FIF bases are designed to capture (\cref{thm:fif-selfaffine}).
At 1000 samples, no model resolves more than 35\% of the variance ($\text{MSE} > 0.6$), and KAN wins by virtue of having fewer parameters (416 vs.\ 840) and a simpler optimization landscape.
This negative result sharpens the regularity-matching claim: it is not roughness \emph{per se} that FI-KAN exploits, but structured roughness with learnable multi-scale correlations.
The rough-coefficient diffusion results ($65$--$79\times$) confirm this distinction: the PDE operator transforms the unstructured fBm coefficient field into a solution with structured, deterministic roughness that FI-KAN captures effectively.

\paragraph{High-dimensional targets.}
Our experiments focus on low-dimensional regression ($d \leq 10$).
The behavior of FI-KAN on high-dimensional targets and the interaction between fractal bases and the curse of dimensionality remain open questions.

\section{Conclusion}
\label{sec:conclusion}

We have introduced Fractal Interpolation KAN (FI-KAN), which incorporates learnable fractal interpolation bases from iterated function system theory into the Kolmogorov--Arnold Network framework.
Two variants, Pure FI-KAN (Barnsley framework) and Hybrid FI-KAN (Navascu\'es framework), provide complementary perspectives on the regularity-matching hypothesis: the geometric regularity of the basis functions should be adapted to the geometric regularity of the target function.

The experimental evidence supports this hypothesis across multiple axes.
On fractal and non-smooth targets, FI-KAN provides up to $6.3\times$ MSE reduction over KAN on genuine fractal targets ($\dimB = 1.5$).
On non-smooth PDE solutions with structured roughness, computed via finite elements (\texttt{scikit-fem}) and exact spectral methods, Hybrid FI-KAN achieves up to $79\times$ improvement over KAN, with the strongest gains on diffusion equations with rough coefficients and corner singularity problems.
The H\"older regularity sweep demonstrates consistent improvements across the full regularity spectrum.
The fractal dimension regularizer provides interpretable, geometry-aware complexity control whose learned values correlate with true target regularity.
An additional advantage in noise robustness ($6.1\times$ on fractal targets) suggests that fractal bases provide structural benefits beyond pure approximation quality.

The contrast between Pure and Hybrid FI-KAN is itself a scientific finding: it demonstrates that basis geometry genuinely matters, and that the optimal choice depends on the regularity of the target.
A further distinction emerges between structured and unstructured roughness: FI-KAN's advantage requires that the target's non-smooth character exhibit learnable multi-scale correlations (as in PDE-inherited roughness or deterministic fractal functions), rather than purely stochastic roughness from single random path realizations.
This establishes regularity-matched basis design as a viable and principled strategy for neural function approximation, opening the door to architectures that explicitly adapt their geometric structure to the problem at hand.
The immediate practical implications lie in scientific domains where non-smooth structure is the rule rather than the exception: subsurface flow in heterogeneous media, elliptic PDEs on non-convex domains, stochastic PDE solutions, geophysical signal processing, and biomedical signals with fractal scaling.
In each of these settings, the ability to learn the appropriate regularity from data, rather than assuming smoothness a priori, addresses a fundamental modeling gap that current neural architectures leave open.

\bibliographystyle{plain}
\bibliography{references}

\FloatBarrier
\appendix

\section{Performance on Smooth Benchmarks}
\label{app:smooth}

For completeness, we evaluate FI-KAN on the canonical toy functions from the KAN paper~\cite{liu2024kan} and on Feynman physics equations from the symbolic regression benchmark~\cite{udrescu2020aifeynman}.
These are smooth, analytic functions where B-spline bases are near-optimal.

\subsection{KAN Paper Toy Functions}

\begin{table}[!htbp]
\centering
\caption{KAN paper toy functions~\cite{liu2024kan}. These are smooth, low-dimensional targets where B-spline KAN is near-optimal. Hybrid FI-KAN wins 1 of 5; Pure FI-KAN wins 0 of 5. This is consistent with the theoretical analysis: fractal bases provide no advantage on smooth analytic targets.}
\label{tab:kan-toys}
\small
\begin{tabular}{@{}l|ccc|ccc@{}}
\toprule
& \multicolumn{3}{c|}{Mean MSE} & \multicolumn{3}{c}{Params} \\
Target & MLP & KAN & Hybrid & MLP & KAN & Hybrid \\
\midrule
$J_0$ Bessel     & 8.62e-2 & \textbf{1.54e-4} & 2.18e-4 & 424 & 208 & 424 \\
$e^x \sin(y) + y^2$ & 7.79e-1 & \textbf{4.32e-2} & 6.16e-2 & 609 & 312 & 608 \\
$x \cdot y$      & \textbf{4.6e-5} & 8.3e-5 & 8.2e-5 & 609 & 312 & 608 \\
10D sum        & 4.57e-3 & 1.21e-2 & \textbf{4.03e-4} & 1081 & 572 & 1080 \\
4D composite   & 2.50 & \textbf{1.38} & 4.27 & 639 & 338 & 652 \\
\bottomrule
\end{tabular}
\end{table}

On smooth, low-dimensional targets ($J_0$ Bessel, $e^x\sin(y)+y^2$, $x \cdot y$), KAN's B-spline bases are near-optimal and the fractal correction provides no advantage.
The one Hybrid FI-KAN win (10D sum) suggests that fractal bases may help in higher dimensions even for relatively smooth targets, possibly by providing additional representational diversity.

Pure FI-KAN performs substantially worse on all five targets (\cref{tab:kan-toys-pure}), consistent with the approximation-theoretic analysis in \cref{sec:theory}.

\begin{table}[!htbp]
\centering
\caption{KAN paper toy functions (Pure FI-KAN).}
\label{tab:kan-toys-pure}
\small
\begin{tabular}{@{}lccc@{}}
\toprule
Target & MLP & KAN & Pure FI-KAN \\
\midrule
$J_0$ Bessel & 8.62e-2 & \textbf{1.54e-4} & 7.30e-3 \\
$e^x\sin(y)+y^2$ & 7.93e-1 & \textbf{4.32e-2} & 1.98 \\
$x \cdot y$ & \textbf{6.1e-5} & 8.3e-5 & 8.22e-4 \\
10D sum & \textbf{8.17e-3} & 1.21e-2 & 3.05e-1 \\
4D composite & 2.75 & \textbf{1.38} & 9.59 \\
\bottomrule
\end{tabular}
\end{table}

\begin{figure}[!htbp]
\centering
\includegraphics[width=\textwidth]{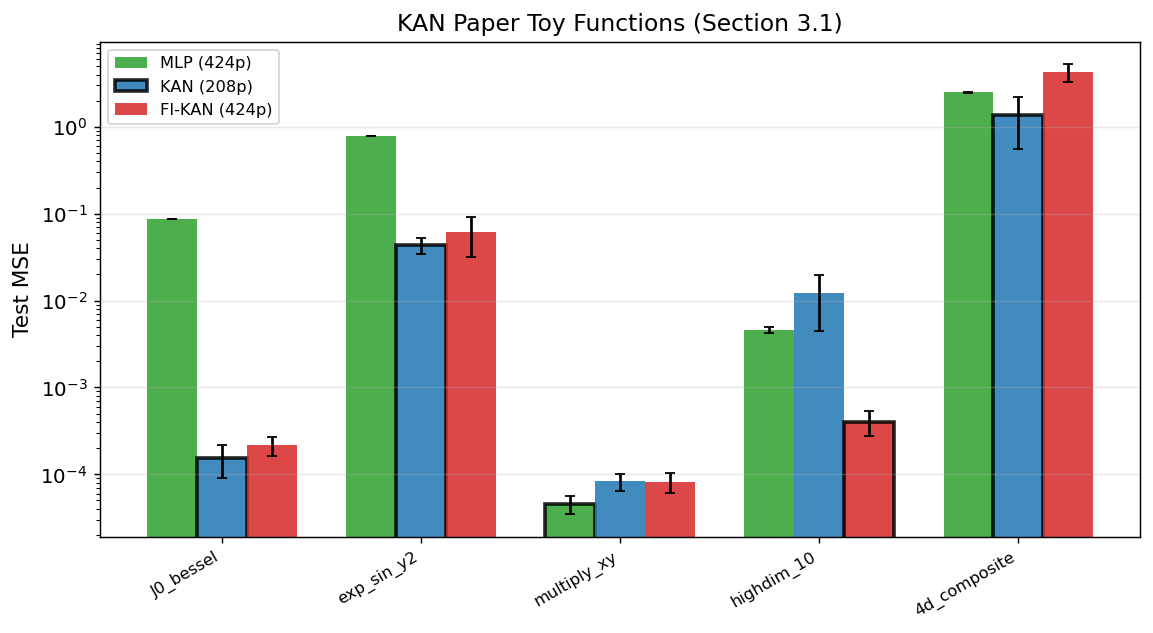}
\caption{KAN paper toy functions (Hybrid FI-KAN). Bar chart comparison of test MSE across five canonical smooth targets. KAN wins on 3 of 5 targets where B-splines are near-optimal. The one Hybrid FI-KAN win (10D sum, $\sim\!30\times$ improvement) suggests that the fractal correction provides additional representational diversity that benefits higher-dimensional targets even when the underlying function is smooth.}
\label{fig:kan-toys}
\end{figure}

\subsection{Feynman Physics Equations}

\begin{table}[!htbp]
\centering
\caption{Feynman physics equations (Hybrid FI-KAN). FI-KAN wins 3 of 7 on these smooth physics targets.}
\label{tab:feynman-hybrid}
\small
\begin{tabular}{@{}lcccl@{}}
\toprule
Equation & MLP & KAN & Hybrid FI-KAN & Best \\
\midrule
I.6.2 Gaussian & 6.46e-1 & 1.32e-2 & \textbf{9.77e-3} & FI-KAN \\
I.12.11 Lorentz & \textbf{9.40e-4} & 2.15e-2 & 4.53e-2 & MLP \\
I.16.6 Relativity & 8.27e-2 & 7.35e-3 & \textbf{1.97e-3} & FI-KAN \\
I.29.16 Distance & 2.00e-2 & \textbf{2.15e-3} & 2.92e-3 & KAN \\
I.30.3 Diffraction & 2.10e-1 & 1.07e-1 & \textbf{5.45e-2} & FI-KAN \\
I.50.26 Cosine & 7.33e-2 & \textbf{1.12e-2} & 1.33e-2 & KAN \\
III.9.52 Sinc & 7.43e-4 & \textbf{2.50e-4} & 4.22e-4 & KAN \\
\bottomrule
\end{tabular}
\end{table}

On smooth physics equations, KAN wins 4 of 7, consistent with the expectation that B-splines are well-adapted to analytic targets.
The FI-KAN wins (Gaussian, relativity, diffraction) occur on functions with sharp transitions or oscillatory structure that benefit from multi-scale representation.

\begin{figure}[!htbp]
\centering
\includegraphics[width=\textwidth]{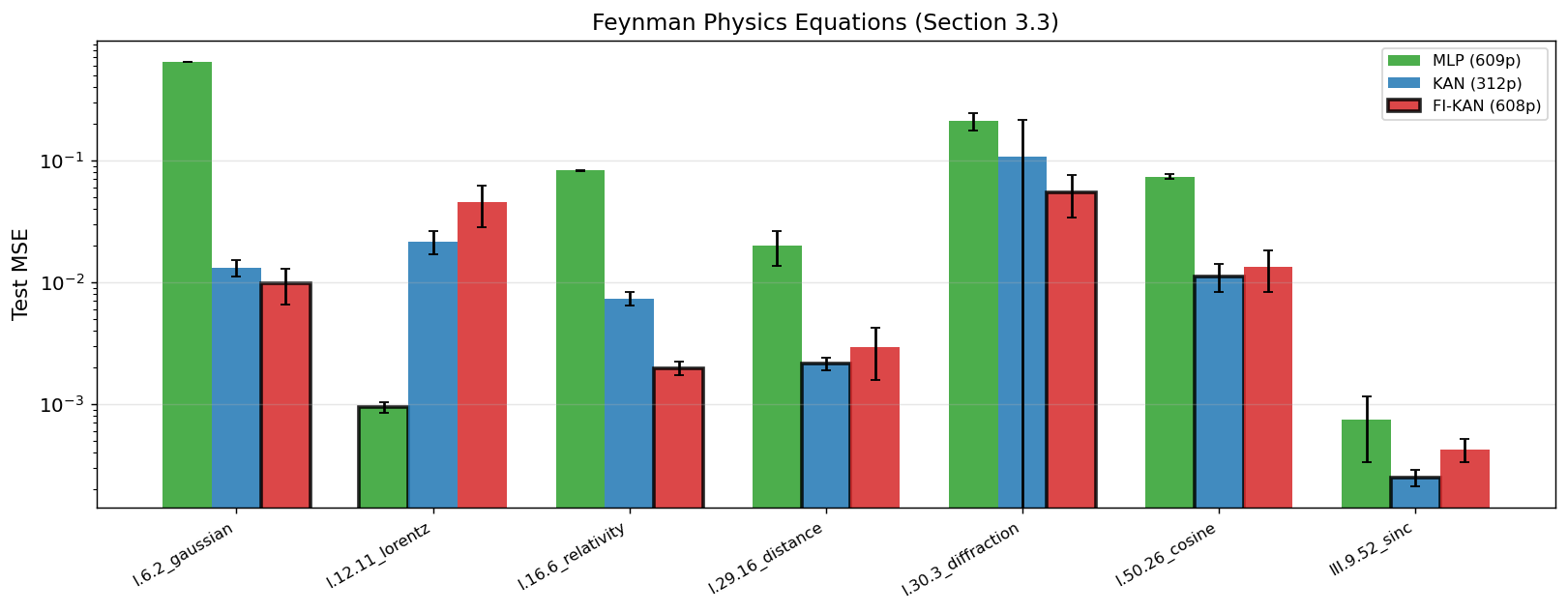}
\caption{Feynman physics equations (Hybrid FI-KAN). Bar chart comparison across seven analytic physics targets. KAN wins 4 of 7, consistent with the expectation that B-splines are well-adapted to analytic targets. The FI-KAN wins (Gaussian, relativity, diffraction) occur on functions with sharp transitions or oscillatory structure that benefit from multi-scale representation.}
\label{fig:feynman}
\end{figure}

\section{Pure FI-KAN Scaling Law Analysis}
\label{app:scaling-pure}

On smooth targets, Pure FI-KAN exhibits negative scaling exponents on both exp\_sin and Weierstrass.
This means test MSE \emph{increases} with parameter count, a pathological behavior explained by \cref{cor:smooth-obstruction}: adding parameters expands the capacity for fractal structure, which is counterproductive when the target is smooth.

\begin{figure}[!htbp]
\centering
\includegraphics[width=\textwidth]{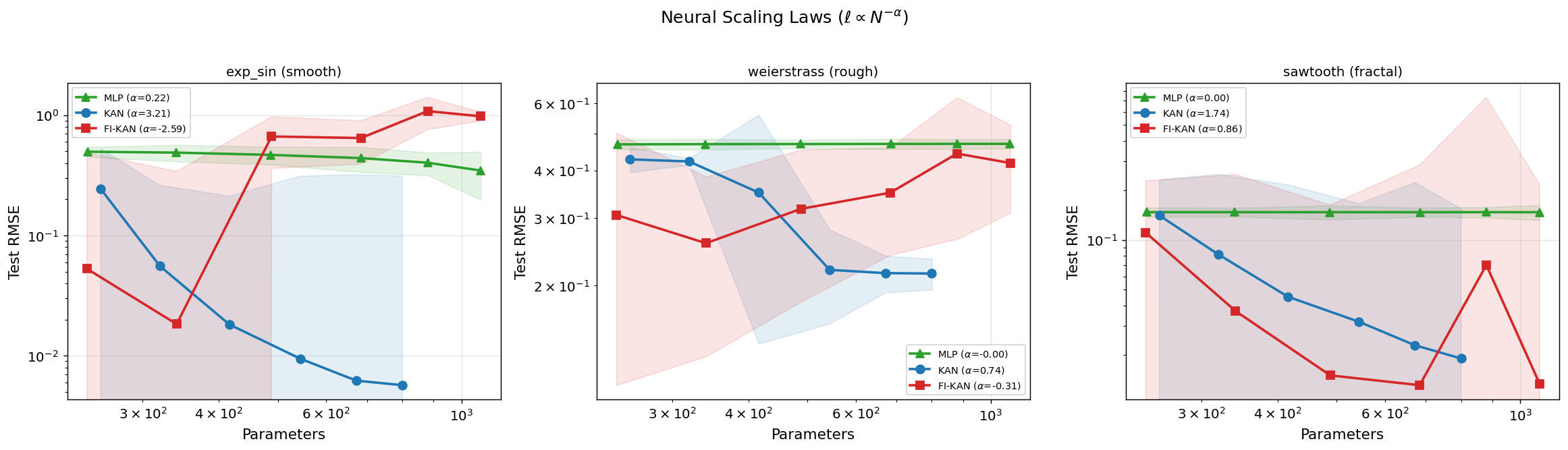}
\caption{Scaling laws (Pure FI-KAN): test RMSE vs.\ parameter count. \textbf{Left (exp\_sin):} Pure FI-KAN exhibits a \emph{negative} scaling exponent: test error \emph{increases} with model size. This is the pathological behavior predicted by \cref{cor:smooth-obstruction}: additional grid points introduce more oscillatory fractal basis functions, making the cancellation required for smooth approximation increasingly fragile. \textbf{Center (Weierstrass):} mild negative scaling. \textbf{Right (sawtooth):} negative scaling at large grid sizes, extending the smooth-approximation obstruction to fractal targets when too many uncontrolled fractal bases are introduced without the spline backbone. At small-to-moderate grid sizes, positive scaling is observed before the pathology sets in.}
\label{fig:scaling-pure}
\end{figure}

On the sawtooth target, Pure FI-KAN exhibits pathological behavior at large grid sizes ($G \geq 16$), with a negative overall scaling exponent.
This extends the smooth-approximation obstruction (\cref{cor:smooth-obstruction}): even on a fractal target, uncontrolled fractal basis growth is destabilizing when too many basis functions are introduced without the spline backbone.
At moderate grid sizes ($G \leq 8$), Pure FI-KAN does achieve positive scaling, confirming that the architecture can productively use moderate additional capacity when the target matches its inductive bias.

These scaling pathologies are not present in the Hybrid variant, where the spline path provides a smooth baseline that cannot be degraded by the fractal correction.

\section{Additional Pure FI-KAN Results}
\label{app:pure-additional}

This appendix collects complementary results for the Pure FI-KAN variant, mirroring the main-body analyses conducted with the Hybrid variant. The consistent pattern across all experiments confirms the regularity-matching hypothesis: Pure FI-KAN excels on rough/fractal targets and struggles on smooth targets, the inverse of standard KAN's behavior.

\begin{figure}[!htbp]
\centering
\includegraphics[width=\textwidth]{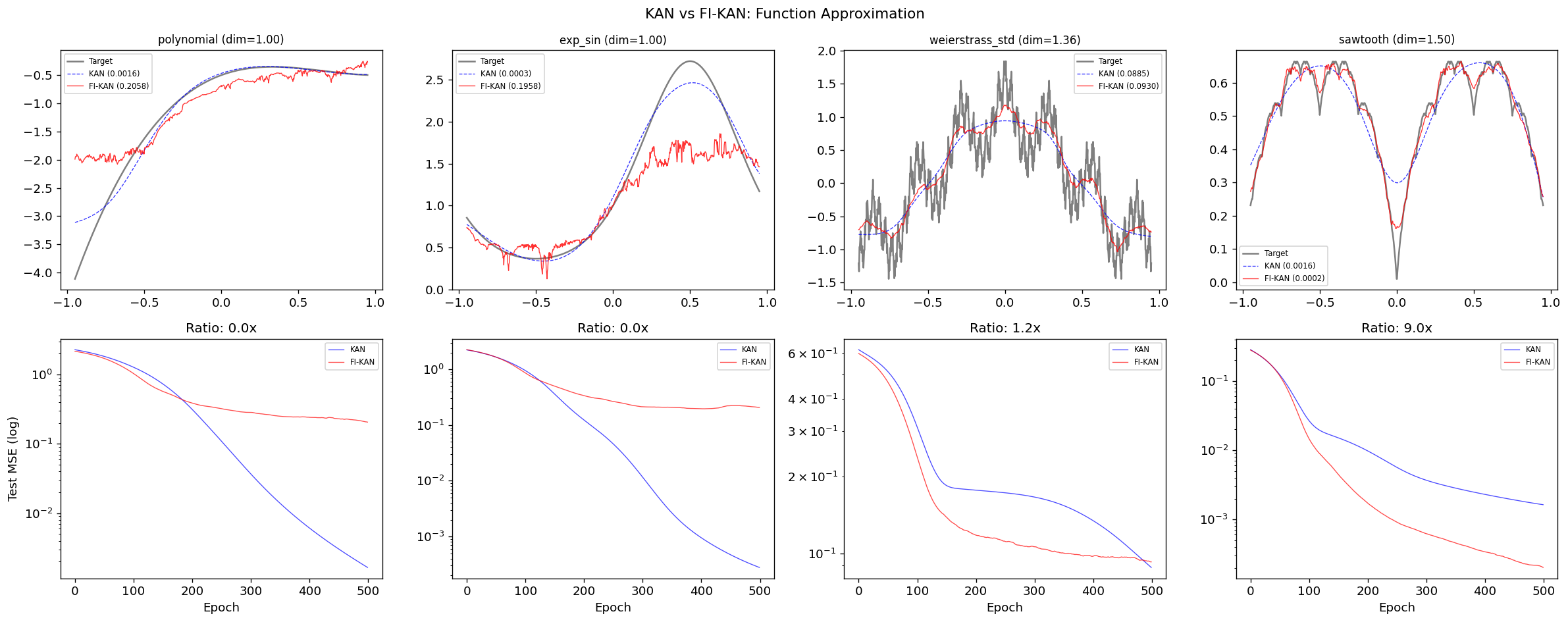}
\caption{Pure FI-KAN vs.\ KAN on four representative targets. \textbf{Top row:} function fits. \textbf{Bottom row:} training curves with improvement ratios. The pattern is complementary to Hybrid FI-KAN (\cref{fig:hybrid-function-approx}): Pure FI-KAN loses badly on smooth targets (ratio 0.0x on polynomial, exp\_sin) but wins on the sawtooth (ratio 2.3x). This asymmetry is the strongest controlled evidence for the regularity-matching hypothesis.}
\label{fig:pure-function-approx}
\end{figure}

\begin{figure}[!htbp]
\centering
\includegraphics[width=\textwidth]{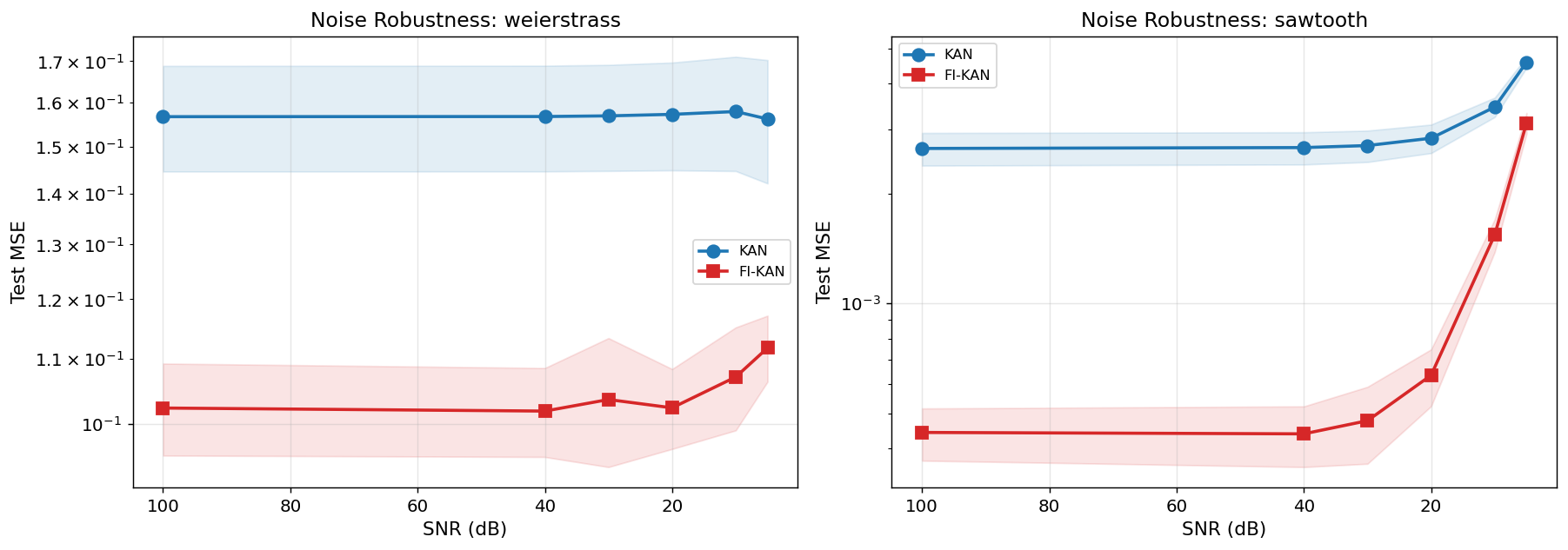}
\caption{Noise robustness (Pure FI-KAN). The Pure variant also outperforms KAN on fractal targets under noise, though with smaller margins than the Hybrid. \textbf{Left (Weierstrass):} $\sim\!1.5\times$ advantage. \textbf{Right (sawtooth):} up to $2.2\times$ advantage at clean SNR.}
\label{fig:noise-pure}
\end{figure}

\begin{figure}[!htbp]
\centering
\includegraphics[width=0.75\textwidth]{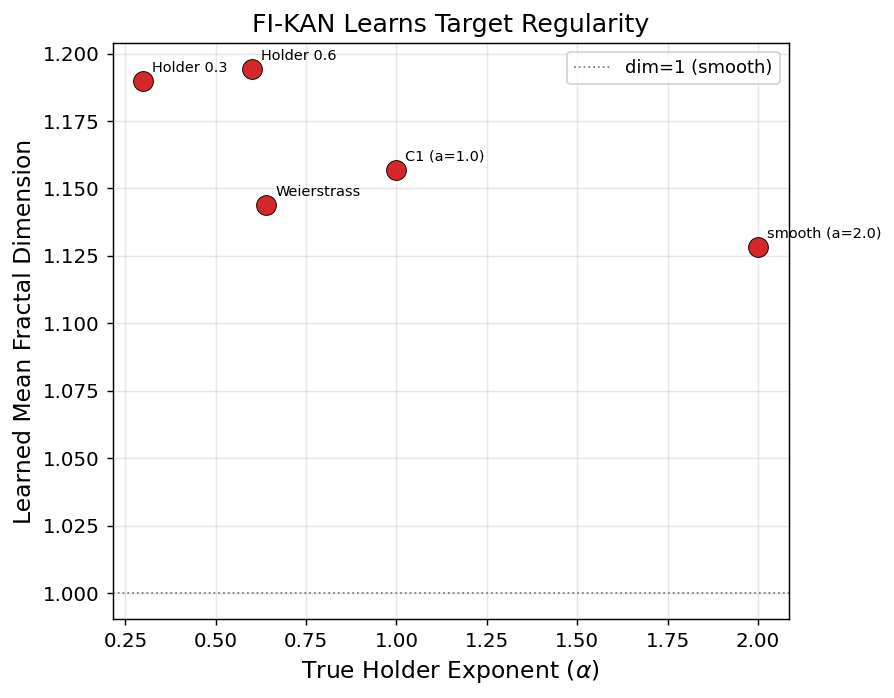}
\caption{Learned fractal dimension vs.\ true H\"older exponent (Pure FI-KAN). The monotonic correlation is present but weaker than the Hybrid variant (\cref{fig:diagnostic-hybrid}): smooth targets receive $\dimB \approx 1.13$ rather than 1.0, because the Pure variant is forced to use some fractal structure even when inappropriate. This systematic overestimation further confirms that fractal-only bases are suboptimal for smooth targets.}
\label{fig:diagnostic-pure}
\end{figure}

\begin{figure}[!htbp]
\centering
\includegraphics[width=\textwidth]{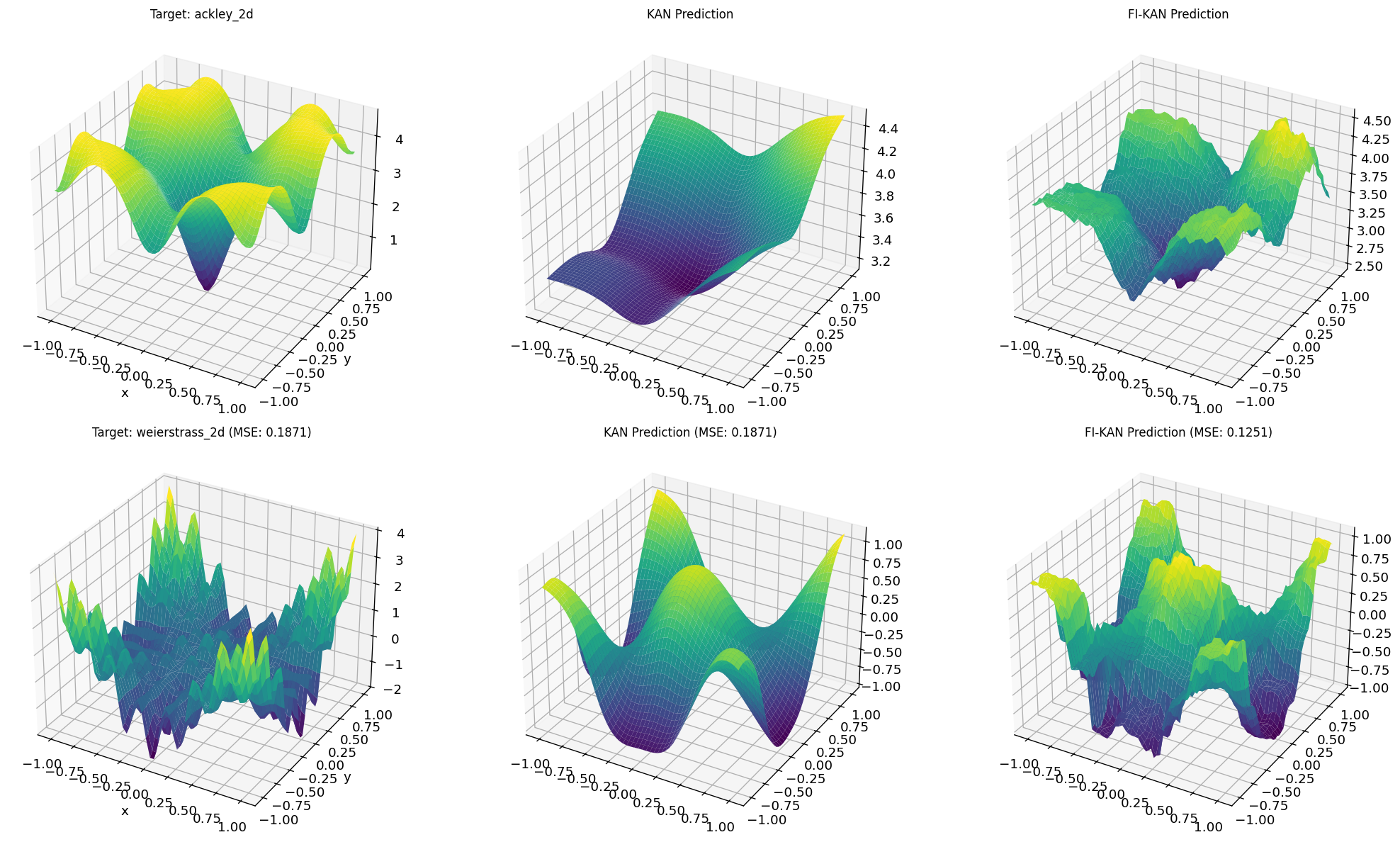}
\caption{2D function regression (Pure FI-KAN). Pure FI-KAN loses on Ackley (smooth outside origin) but captures more fine-scale structure on the 2D Weierstrass target than KAN (MSE 0.1251 vs.\ 0.1871), consistent with the regularity-matching thesis.}
\label{fig:2d-pure}
\end{figure}

\end{document}